\documentclass[draftcls,onecolumn]{IEEEtran}

\usepackage{xr}
\usepackage{cite}
\usepackage{amsmath}
\usepackage[colorlinks=true,linkcolor=black,citecolor=black,urlcolor=black]{hyperref}
\usepackage{cases}
\usepackage[utf8]{inputenc}
\usepackage[english]{babel}
\usepackage{footnote}
\usepackage{bbm}
\usepackage[T1]{fontenc}    % use 8-bit T1 fonts
\usepackage{url}            % simple URL typesetting
\usepackage{booktabs}       % professional-quality tables
\usepackage{amsfonts}       % blackboard math symbols
\usepackage{nicefrac}       % compact symbols for 1/2, etc.
\usepackage{microtype}      % microtypography
\usepackage{rotating}

\usepackage[justification=centering]{caption}

\usepackage{bm}

\usepackage[mathscr]{eucal}
\usepackage{epsfig,epsf,psfrag}
\usepackage{amssymb,amsmath,amsfonts,latexsym}
\usepackage{amsmath,graphicx,bm,xcolor,url}
\usepackage[caption=false]{subfig} 
\usepackage{fixltx2e}%ordering of single and double column floats
\usepackage{array}%array and tabular environments
\usepackage{verbatim}
\usepackage{bm}
\usepackage{algpseudocode}
\usepackage{algorithm}
\usepackage{verbatim}
\usepackage{textcomp}
\usepackage{mathrsfs}
\usepackage{epstopdf}
\usepackage{relsize}
\usepackage{cleveref} 
\usepackage{subfig}
 \usepackage{amsthm}

\catcode`~=11 \def\UrlSpecials{\do\~{\kern -.15em\lower .7ex\hbox{~}\kern .04em}} \catcode`~=13 

\allowdisplaybreaks[3]

\newcommand{\calA}{\mathcal{A}}

\newcommand{\calE}{\mathcal{E}}

\newcommand{\calK}{\mathcal{K}}

\newcommand{\calN}{\mathcal{N}}

\newcommand{\calP}{\mathcal{P}}

\newcommand{\calS}{\mathcal{S}}

\newcommand{\ba}{\mathbf{a}}
\newcommand{\bA}{\mathbf{A}}
\newcommand{\bb}{\mathbf{b}}
\newcommand{\bB}{\mathbf{B}}

\newcommand{\bC}{\mathbf{C}}

\newcommand{\bD}{\mathbf{D}}

\newcommand{\bF}{\mathbf{F}}
\newcommand{\bg}{\mathbf{g}}

\newcommand{\bH}{\mathbf{H}}

\newcommand{\bI}{\mathbf{I}}

\newcommand{\bp}{\mathbf{p}}

\newcommand{\bq}{\mathbf{q}}

\newcommand{\br}{\mathbf{r}}
\newcommand{\bR}{\mathbf{R}}
\newcommand{\bs}{\mathbf{s}}

\newcommand{\bt}{\mathbf{t}}

\newcommand{\bu}{\mathbf{u}}

\newcommand{\bv}{\mathbf{v}}

\newcommand{\bw}{\mathbf{w}}
\newcommand{\bW}{\mathbf{W}}
\newcommand{\bx}{\mathbf{x}}

\newcommand{\by}{\mathbf{y}}

\newcommand{\bz}{\mathbf{z}}

\newcommand{\rmd}{\mathrm{d}}

\newcommand{\rmF}{\mathrm{F}}

\newcommand{\bbC}{\mathbb{C}}

\newcommand{\bbE}{\mathbb{E}}

\newcommand{\bbP}{\mathbb{P}}

\newcommand{\bbR}{\mathbb{R}}

\DeclareMathAlphabet{\mathbsf}{OT1}{cmss}{bx}{n}
\DeclareMathAlphabet{\mathssf}{OT1}{cmss}{m}{sl}% slanted sans serif

\DeclareSymbolFont{bsfletters}{OT1}{cmss}{bx}{n}  
\DeclareSymbolFont{ssfletters}{OT1}{cmss}{m}{n}
\DeclareMathSymbol{\bsfGamma}{0}{bsfletters}{'000}
\DeclareMathSymbol{\ssfGamma}{0}{ssfletters}{'000}
\DeclareMathSymbol{\bsfDelta}{0}{bsfletters}{'001}
\DeclareMathSymbol{\ssfDelta}{0}{ssfletters}{'001}
\DeclareMathSymbol{\bsfTheta}{0}{bsfletters}{'002}
\DeclareMathSymbol{\ssfTheta}{0}{ssfletters}{'002}
\DeclareMathSymbol{\bsfLambda}{0}{bsfletters}{'003}
\DeclareMathSymbol{\ssfLambda}{0}{ssfletters}{'003}
\DeclareMathSymbol{\bsfXi}{0}{bsfletters}{'004}
\DeclareMathSymbol{\ssfXi}{0}{ssfletters}{'004}
\DeclareMathSymbol{\bsfPi}{0}{bsfletters}{'005}
\DeclareMathSymbol{\ssfPi}{0}{ssfletters}{'005}
\DeclareMathSymbol{\bsfSigma}{0}{bsfletters}{'006}
\DeclareMathSymbol{\ssfSigma}{0}{ssfletters}{'006}
\DeclareMathSymbol{\bsfUpsilon}{0}{bsfletters}{'007}
\DeclareMathSymbol{\ssfUpsilon}{0}{ssfletters}{'007}
\DeclareMathSymbol{\bsfPhi}{0}{bsfletters}{'010}
\DeclareMathSymbol{\ssfPhi}{0}{ssfletters}{'010}
\DeclareMathSymbol{\bsfPsi}{0}{bsfletters}{'011}
\DeclareMathSymbol{\ssfPsi}{0}{ssfletters}{'011}
\DeclareMathSymbol{\bsfOmega}{0}{bsfletters}{'012}
\DeclareMathSymbol{\ssfOmega}{0}{ssfletters}{'012}

\newcommand{\balpha}{\bm{\alpha}}

\newcommand{\bepsilon}{\bm{\epsilon}}

\newcommand{\bxi}{\bm{\xi}}

\newcommand{\bSigma	}{\bm{\Sigma}}

\newcommand{\bPhi}{\bm{\Phi}}

\DeclareMathOperator{\sech}{sech}

\theoremstyle{plain}
\newtheorem{theorem}{Theorem} 
\newtheorem{lemma}{Lemma}
\newtheorem{proposition}{Proposition}
\newtheorem{corollary}{Corollary}
\newtheorem{definition}{Definition} 

\newtheorem{remark}{Remark}

\newcommand{\qednew}{\nobreak \ifvmode \relax \else
      \ifdim\lastskip<1.5em \hskip-\lastskip
      \hskip1.5em plus0em minus0.5em \fi \nobreak
      \vrule height0.75em width0.5em depth0.25em\fi}

\newcommand\sign[1]{\mathrm{sign}{#1}}

\begin{document}
    
\title{Robust 1-bit Compressive Sensing with Partial Gaussian Circulant Matrices and Generative Priors}

\author{Zhaoqiang Liu, Subhroshekhar Ghosh, Jun Han, Jonathan Scarlett

\thanks{
Z.~Liu is with the Department of Computer Science, National University of Singapore (email: \url{zqliu12@gmail.com}, \url{dcslizha@nus.edu.sg}).

S.~Ghosh is with the Department of Mathematics, National University of Singapore (email: \url{subhrowork@gmail.com}).

J.~Han is with the Platform and Content Group, Tencent (email: \url{jim.han665@gmail.com}).

J.~Scarlett is with the Department of Computer Science, the Department of Mathematics and the Institute of Data Science, National University of Singapore (email: \url{scarlett@comp.nus.edu.sg}).

This work was supported by the Singapore National Research Foundation (NRF) under grant R-252-000-A74-281, and the Singapore Ministry of Education (MoE) under grants R-146-000-250-133 and R-146-000-312-114.}}

\maketitle

\begin{abstract}
    In 1-bit compressive sensing, each measurement is quantized to a single bit, namely the sign of a linear function of an unknown vector, and the goal is to accurately recover the vector. While it is most popular to assume a standard Gaussian sensing matrix for 1-bit compressive sensing, using structured sensing matrices such as partial Gaussian circulant matrices is of significant practical importance due to their faster matrix operations. In this paper, we provide recovery guarantees for a correlation-based optimization algorithm for robust 1-bit compressive sensing with randomly signed partial Gaussian circulant matrices and generative models. Under suitable assumptions, we match guarantees that were previously only known to hold for i.i.d.~Gaussian matrices that require significantly more computation. We make use of a practical iterative algorithm, and perform numerical experiments on image datasets to corroborate our theoretical results.
\end{abstract}

 \section{Introduction}
\label{sec:intro}

%\vspace{-.1in}
The goal of compressive sensing (CS)~\cite{Fou13,wainwright2019high} is to recover a signal vector from a small number of linear measurements, by making use of prior knowledge on the structure of the vector in the relevant domain. The structure is typically represented by sparsity in a suitably-chosen basis, or generative priors motivated by recent success of deep generative models~\cite{bora2017compressed}. The CS problem has been popular over the past 1--2 decades, and there is a vast amount of literature with theoretical guarantees, including sharp performance bounds for both practical algorithms~\cite{Wai09a,Don13,Ame14,Wen2016} and potentially intractable information-theoretically optimal algorithms~\cite{Wai09,Ari13,Can13,Sca15,scarlett2019introductory,liu2020information,kamath2020power,liu2021towards}.

For the sake of low-cost implementation in hardware and robustness to certain nonlinear distortions~\cite{boufounos2010reconstruction}, instead of assuming infinite-precision real-valued measurements as in CS, it has also been popular to consider 1-bit CS~\cite{boufounos20081,gupta2010sample,zhu2015towards,zhang2014efficient,gopi2013one}, in which the unknown vector is accurately recovered from binary-valued measurements. 

\iffalse
Most of the relevant work for 1-bit CS has focused on standard Gaussian sensing matrices, and unlike linear CS, 1-bit CS may easily fail when using a sensing matrix that is not standard Gaussian~\cite{ai2014one}. However, in some real-world applications, including radar imaging, Fourier optical imaging, and channel estimation, instead of using a standard Gaussian sensing matrix, it is more natural to consider a subsampled Gaussian circulant sensing matrix~\cite{romberg2009compressive}. In addition, based on the fast Fourier transform (FFT), using a partial Gaussian circulant matrix results in fast matrix-vector multiplications~\cite{yu2017binary,gray2005toeplitz}. Motivated by recent progress on 1-bit CS with partial Gaussian circulant matrices~\cite{dirksen2017one,dirksen2018robust} and successful applications of deep generative models, we provide recovery guarantees for a correlation-based optimization algorithm for robust 1-bit CS with partial Gaussian circulant matrices and generative priors. 
\fi

Most of the relevant work for 1-bit CS has focused on standard Gaussian sensing matrices, and unlike linear CS, 1-bit CS may easily fail when using a sensing matrix that is not standard Gaussian~\cite{ai2014one}.  However, since such matrices are unstructured, the associated matrix operations (e.g., multiplication) used in the decoding procedure can be prohibitively expensive when the problem size is large.  This motivates the use of {\em structured} matrices permitting fast matrix operations, and a notable example is 
%However, in some real-world applications, including radar imaging, Fourier optical imaging, and channel estimation, instead of using a standard Gaussian sensing matrix, it is more natural to consider 
a subsampled Gaussian circulant sensing matrix~\cite{romberg2009compressive}, which can be implemented via the fast Fourier transform (FFT)~\cite{yu2017binary,gray2005toeplitz}. Motivated by recent progress on 1-bit CS with partial Gaussian circulant matrices~\cite{dirksen2017one,dirksen2018robust} and successful applications of deep generative models, we provide recovery guarantees for a correlation-based optimization algorithm for robust 1-bit CS with partial Gaussian circulant matrices and generative priors. 

%\medskip

\subsection{Related Work}

It is established in~\cite{foucart2016flavors} that an $\ell_1/\ell_2$-restricted isometry property (RIP) of the sensing matrix guarantees accurate uniform recovery for sparse vectors via a hard thresholding algorithm and a linear programming algorithm. Based on this result, the authors of~\cite{dirksen2017one} provide recovery guarantees for noiseless 1-bit CS with partial Gaussian circulant matrices by proving the required RIP. In particular, they show that %in a small sparsity regime, %and for a sufficiently small $\epsilon$, 
$O\left(\frac{s}{\epsilon^4} \log \left(\frac{n}{\epsilon s}\right)\right)$ noiseless measurements suffice for the reconstruction of the direction of any $s$-sparse vector in $n$-dimensional space up to accuracy $\epsilon$ if the sparsity level $s$ satisfies $s = O\big(\sqrt{\frac{\epsilon n}{\log n}}\big)$. The authors also consider a dithering technique~\cite{knudson2016one,xu2020quantized,jacques2017time,dirksen2018non} to enable the estimation of the norm of the signal vector. 

The authors of~\cite{dirksen2018robust} establish an optimal (up to logarithmic factors) sample complexity for robust 1-bit sparse recovery using a correlation-based algorithm with $\ell_2$-norm regularization.  The setup and results are related to ours, but bear several important differences, including (i) the use of regularization; (ii) the focus on sub-Gaussian noise added before quantization; (iii) the use of dithering-type thresholds in the measurements; (iv) the consideration of sparse signals instead of generative priors; and (v) the use of a random (rather than fixed) index set in the partial circulant matrix.

Some theoretical guarantees for circulant binary embeddings, which are closely related to 1-bit CS with partial Gaussian circulant matrices, have been presented in~\cite{yi2015binary,yu2017binary,oymak2017near,dirksen2018fast}. On the other hand, based on successful applications of deep generative models, instead of assuming the signal vector is sparse, it has been recently popular to assume that the signal vector lies in the range of a generative model in the problem of CS~\cite{bora2017compressed,hand2018phase,van2018compressed,heckel2019deep,hand2019global,ongie2020deep,jalal2020robust,whang2020compressed,cocola2020nonasymptotic}. In particular, 1-bit CS with generative priors have also been studied in~\cite{liu2020sample,qiu2020robust,liu2020generalized}, but none of these papers considers structured sensing matrices that permit fast matrix operations. %partial Gaussian circulant matrices. %We make use of the algorithm proposed in~\cite{zhang2014efficient}, which is a correlation-based algorithm with an $\ell_1$-norm regularization on the vector to optimize. The authors of~\cite{zhang2014efficient} assume a standard Gaussian sensing matrix, and show that to achieve an $\epsilon$-accurate recovery, $O\left(\frac{s}{\epsilon^2} \log n\right)$ samples are sufficient for $s$-sparse vectors, and $O\left(\frac{s}{\epsilon^4} \log n\right)$ samples are sufficient for approximately $s$-sparse vectors. 

\subsection{Contribution}

We consider the problem of recovering a signal vector in the range of a generative model via noisy 1-bit measurements using a correlation-based optimization objective and a randomly signed partial Gaussian circulant sensing matrix. We characterize the number of measurements sufficient to attain an accurate recovery of an underlying $\ell_\infty$-norm bounded signal\footnote{While the assumption about $\ell_\infty$ norm bound appears to be restrictive, we may remove this assumption by multiplying a randomly signed Hadamard matrix for pre-processing. See Remarks~\ref{remark:v1} and~\ref{remark:v2} for detalied explanations.} in the range of a Lipschitz continuous generative model. In particular, for an $L$-Lipschitz continuous generative model with bounded $k$-dimensional input and the range containing in the unit sphere, we show that roughly $O\big(\frac{k \log L}{\epsilon^2}\big)$ samples suffice for $\epsilon$-accurate recovery. This matches the scaling obtained previously for i.i.d.~Gaussian matrices \cite{liu2020sample}. We also consider the impact of adversarial corruptions on the measurements, and show that our scheme is robust to adversarial noise. We make use of a practical iterative algorithm proposed in~\cite{liu2020sample}, and perform some simple numerical expriments on real datasets to support our theoretical results.

\subsection{Notation}

We use upper and lower case boldface letters to denote matrices and vectors respectively. We write $[N]=\{0,1,2,\cdots,N-1\}$ for a positive integer $N$. For any vector $\bt = [t_0,t_1,\ldots,t_{n-1}]^T \in \bbC^n$, we let $\bC_\bt \in \bbC^{n \times n}$ be a circulant matrix corresponding to $\bt$. That is, $\bC_\bt(i,j) = t_{j-i}$, for all $i,j \in [n]$, where here and subsequently, $j-i$ is understood modulo $n$. Let $\bF \in \bbC^{n\times n}$ be the (normalized) Fourier matrix, i.e., $\bF(i,j) = \frac{1}{\sqrt{n}} e^{-\frac{2\pi ij\sqrt{-1}}{n}}$, for all $i,j \in [n]$.  We define the $\ell_2$-ball $B_2^n(r):=\{\bx \in \bbR^n: \|\bx\|_2 \le r\}$, %, and and we use $B_2^n$ to abbreviate the unit ball $B_2^n(1)$. 
and we use $\calS^{n-1} := \{\bx \in \bbR^n: \|\bx\|_2=1\}$ to denote the unit sphere in $\bbR^n$. %For any $\bx \in \bbR^n$, $\|\bx\|_0:=|\{i \in [n]\,:\, x_i \ne 0\}|$ denotes the number of non-zero entries of $\bx$. A vector $\bx$ is said to be $s$-sparse if $\|\bx\|_0 \le s$, and approximately $s$-sparse if $\frac{\|\bx\|_1}{\|\bx\|_2} \le \sqrt{s}$.  
For a vector $\bx =[x_0,x_1,\ldots,x_{n-1}]^T \in \bbR^n$ and an index $i$, we denote by $s_{\rightarrow i}(\bx)$ the vector shifted to right by $i$ positions, i.e., the $j$-th entry of $s_{\rightarrow i}(\bx)$ is the $((j - i) \text{ mod } n)$-th entry of $\bx$. Similarly, we denote by $s_{i\leftarrow }(\bx)$ the vector shifted to left by $i$ positions. A Rademacher variable takes the values $+1$ or $-1$, each with probability $\frac{1}{2}$, and a sequence of independent Rademacher variables is called a Rademacher sequence. The operator norm of a matrix from $\ell_p$ into $\ell_p$ is defined as $\|\bA\|_{p \to p} := \max_{\|\bx\|_p =1} \|\bA\bx\|_p$. We use $g$ to represent a standard normal random variable, i.e., $g \sim \calN(0,1)$. The symbols $C, C', C'', c$ are absolute constants whose values may differ from line to line. Some additional notations that are used specifically for the proof of Lemma~\ref{lem:essential_circ} are presented in a table in Appendix~\ref{sec:table_proof_lemma1}.

\section{Problem Setup} \label{sec:setup}

% In this section, we formally introduce the problem, and overview the main assumptions that we adopt. 
Except where stated otherwise, we consider the following setup and assumptions:
\begin{itemize}
 \item Let $\bx^* =[x_0^*,x_1^*,\ldots,x_{n-1}^*]^T \in \bbR^n$ be the underlying signal vector, assumed to have unit norm (since recovery of the norm is impossible under our measurement model). We write $\tilde{\bx}^*$ as $ \tilde{\bx}^* = \bD_{\bxi} \bx^*$, where $\bD_{\bxi}$ is a diagonal matrix corresponding to a Rademacher sequence $\bm{\xi} = [\xi_0,\xi_1,\ldots,\xi_{n-1}]^T \in \bbR^n$.
 
 \item Fixing any (deterministic) index set $\Omega \subseteq [n]$ with $|\Omega| = m$, we set the sensing matrix $\bA \in \bbR^{m \times n}$ as 
\begin{equation}\label{eq:partial_circ_sign}
 \bA = \bR_{\Omega} \bC_{\bg} \bD_{\bxi},
\end{equation}
where $\bR_{\Omega} \in \bbR^{m \times n}$ is the matrix that restricts an $n$-dimensional vector to coordinates indexed by $\Omega$ and $\bg =[g_0,g_1,\ldots,g_{n-1}]^T \sim \calN(0,\bI_n)$. %and $\bD_{\bxi}$ is a diagonal matrix corresponding to a Rademacher sequence $\bm{\xi} = [\xi_0,\xi_1,\ldots,\xi_{n-1}]^T \in \bbR^n$. 
That is, $\bA$ is a partial Gaussian circulant matrix with random column sign flips. For concreteness, we fix the index set $\Omega$ as $\Omega = [m]$, but the same analysis holds for any other fixed set. We also write $\tilde{\bA}$ as $\tilde{\bA} = \bR_{\Omega} \bC_{\bg} $, i.e., a partial Gaussian circulant matrix without random column sign flips.
 
 \item We assume that the measurements, or response variables, $b_i \in \{-1,1\}, i \in [m]$, are generated according to
\begin{equation}
 b_i = \theta\left(\left\langle \ba_i, \bx^*\right\rangle\right),
\end{equation}
where $\ba_i^T$ is the $i$-th row of the sensing matrix $\bA \in \bbR^{m \times n}$, and $\theta$ is a (possibly random) function with $\theta (z) \in \{-1,+1\}$. We focus on the following widely-adopted noisy 1-bit observation models:
\begin{itemize}
 \item {\em Random bit flips:} $\theta(x) = \tau \cdot \sign(x)$, where $\tau$ is a $\pm 1$-valued random variable with $\bbP(\tau = -1) = p < \frac{1}{2}$, independent of $\bA$.
 
 \item {\em Random noise before quantization}: $\theta(x) = \sign(x + e)$, where $e$ represents a random noise term added before quantization, independent of $\bA$. For example, when $e\sim \calN(0,\sigma^2)$, this corresponds to a special case of the probit model; when $e$ is a logit noise term, this recovers the logistic regression model. % For this $\theta$, we have $\lambda = \bbE[\theta(g)g] = \sqrt{\frac{2}{\pi (1+\sigma^2)}}$~\cite{plan2012robust}.
 
 %\item {\em Logit noise:} $\theta(x) = \sign(x + e)$, where $e $ represents a logit noise term added before quantization. This recovers the logistic regression model. 
\end{itemize}

%$\theta$ may be unknown or unspecified. 

\item We assume that
\begin{equation}
 \bbE [\theta(g) g] = \lambda
\end{equation}
for some $\lambda > 0$, where $g$ is a standard normal random variable. Note that the expectation is taken with respect to both $\theta$ and $g$. Then, it is easy to calculate that $\bbE [b_i \ba_i ] = \lambda \bx^*$~\cite{zhang2014efficient}. For the above choices of $\theta$, we have~\cite[Sec.~3]{plan2012robust}:
\begin{itemize}
 \item {\em Random bit flips:} $\lambda = (1-2p)\sqrt{\frac{2}{\pi}}$.
 \item {\em Random noise before quantization}: For any $x \in \bbR$, we set 
\begin{equation}
 \phi(x) = \bbE[\theta(x)] = 1-2 \bbP(e < -x).
\end{equation}
%where $\bbE_e$ denotes that the expectation is taken over $e$.
Then, we have $\phi'(x) = 2f_e(-x)$, where $f_e$ is the probability density function (pdf) of $e$. Thus, we have  
\begin{align}
 \lambda = \bbE[\theta(g)g] = \bbE[\bbE[\theta(g)g|g]] = \bbE[\phi(g)g] = \bbE[\phi'(g)] = 2\bbE[f_e(-g)] >0.
\end{align}
In particular, when $e\sim \calN(0,\sigma^2)$, we have $f_e(x) = \frac{1}{\sqrt{2\pi}\sigma} e^{-\frac{x^2}{2\sigma^2}}$ and $\lambda = \sqrt{\frac{2}{\pi (1+\sigma^2)}}$. When $e$ is a logit noise term, we have $\phi(x):=\tanh \big(\frac{x}{2}\big)$, and $\lambda = \bbE[\phi'(g)] = \frac{1}{2}\bbE\big[\sech^2\big(\frac{g}{2}\big)\big]$.
\end{itemize}

\item For the case of random noise added before quantization, we additionally make the following assumptions:
\begin{itemize}
 \item Let $F_e$ be the cumulative distribution function (cdf) of $e$. We assume that $F_e(x) + F_e(-x) = 1$ for any $x\in\bbR$.\footnote{Or equivalently, $f_e(x) = f_e(-x)$ for any $x \in \bbR$.}%, where $f_e$ is the probability density function (pdf) of $e$.} 
 \item For any $a>0$, let $\theta_a(x) = \sign(x+\frac{e}{a})$. We assume that $\lambda_a := \bbE[\theta_a(g) g] < C_1 a + C_2$, where $g \sim\calN(0,1)$ and $C_1,C_2$ are absolute constants.  
\end{itemize}
The assumption $F_e(x) + F_e(-x) = 1$ ensures that $\sign(g+\frac{e}{a})$ is a Rademacher random variable. We show in Lemma~\ref{lem:add_assumps_e} that the above assumptions are satisfied when $e$ corresponds to the Gaussian or logit noise term. 

\item In addition to any random noise in $\theta$, we allow for adversarial noise.  In this case, instead of observing $\bb=[b_0,\ldots,b_{m-1}]^T\in\bbR^m$ directly, we only assume access to $\tilde{\bb} = [\tilde{b}_0,\ldots,\tilde{b}_{m-1}]^T \in \bbR^m$ satisfying
 \begin{equation}
    \frac{1}{\sqrt{m}} \|\tilde{\bb}-\bb\|_2 \le \varsigma
 \end{equation}
 for some parameter $\varsigma \ge 0$.  Note that the corruptions of $\bb$ yielding $\tilde{\bb}$ may depend on $\bA$.

\item The signal vector $\bx^*$ lies in a set of structured signals $\calK$. We focus on the case that $\calK = \mathrm{Range}(G)$ for some $L$-Lipschitz continuous generative model $G \,:\, B_2^k(r) \rightarrow \bbR^n$ (e.g., see \cite{bora2017compressed}). We will also assume that $\mathrm{Range}(G) \subseteq \calS^{n-1}$, but this can easily be generalized 
%Such an assumption is commonly made for 1-bit compressive sensing (e.g., see \cite{plan2012robust}). Moreover, it is easy to be generalized 
to the case that $\mathrm{Range}(G) \subseteq \bbR^n$ and the recovery guarantee is only given up to scaling \cite{liu2020sample,liu2020generalized}.

\item To derive an estimate of $\bx^*$, we seek $\hat{\bx}$ maximizing a constrained correlation-based objective function:
\begin{equation}\label{eq:opt_corr}
 \hat{\bx} := \arg\max_{\bx \in \calK} \quad \tilde{\bb}^T(\bA \bx).
\end{equation}
Such an optimization problem has been considered in~\cite{plan2012robust}, with the main purpose of deriving a convex optimization problem for the case that $\calK$ is a convex set. We focus on the case that $\calK = \mathrm{Range}(G)$, which in general leads to a non-convex optimization problem. This can be approximately solved using gradient-based algorithms~\cite{liu2020sample} or by projecting $\bA^T\tilde{\bb}$ onto $\calK$~\cite{jacques2013quantized,plan2017high}. In this paper, we provide theoretical guarantees for the recovery in the ideal case that the global maximizer is found. We also provide experimental results in Section~\ref{sec:exp}.
\end{itemize}

\section{Main Results}
\label{sec:main_res}

Recall that the $i$-th observation is $b_i = \theta(\langle \ba_i,\bx^*\rangle)$, $\tilde{\bx}^* = \bD_{\bxi}\bx^*$, and $\tilde{\bA}= \bR_{\Omega} \bC_{\bg}$, i.e., a partial Gaussian circulant matrix without random column sign flips. We have the following important lemma. 
\begin{lemma}\label{lem:essential_circ}
For any $t> 0$ satisfying $m =\Omega\left(t + \log n\right)$, suppose that the unit vector $\bx^*$ satisfies $\|\bx^*\|_\infty \le \rho$ with $\rho = O\big(\frac{1}{m (t+\log m)}\big)$. Suppose that $\bb = \theta(\bA \bx^*)$ with $\theta(x) = \tau \cdot \sign(x)$ or $\theta(x) = \sign(x + e)$ ({\em cf.} Sec.~\ref{sec:setup}), and $\bA$ is a partial Gaussian circulant matrix with random column sign flips. %We define $\lambda := \bbE[\theta(g) g]$ with $g \sim \calN(0,1)$. 
Then, with probability at least $1-e^{-\Omega(t)} - e^{-\Omega(\sqrt{m})}$,
 \begin{equation}
  \left\|\frac{1}{m}\tilde{\bA}^T \bb  -\lambda \tilde{\bx}^*\right\|_\infty \le O\left(\sqrt{\frac{t+\log n }{m}}\right).
 \end{equation}
\end{lemma}

Based on Lemma~\ref{lem:essential_circ}, and following the proof techniques used in~\cite{liu2020generalized} (see Section~\ref{sec:proof_main_theorem} for the details), we have the following theorem concerning the recovery guarantee for noisy 1-bit CS with partial Gaussian circulant matrices and generative priors.
\begin{theorem}
\label{thm:corr_circulant}
Consider any $\bx^* \in \calK$ with $\calK = G(B_2^k(r))$ for some $L$-Lipschitz generative model $G \,:\, B_2^k(r) \to \calS^{n-1}$. Fix $\delta >0$ satisfying $\delta = O\big(\frac{\sqrt{k\log\frac{Lr}{\delta}}}{\lambda\sqrt{m}}+\frac{\varsigma}{\lambda}\big)$, and suppose that $Lr = \Omega(\delta n)$, $n =\Omega\big(m\big)$, %$m = \Omega\big(k \log\frac{Lr}{\delta}\big)$, 
and the signal $\bx^*$ satisfies $\|\bx^*\|_\infty \le \rho = O\big(\frac{1}{m k \log \frac{Lr}{\delta}}\big)$. Suppose that all other conditions are the same as those in Lemma~\ref{lem:essential_circ}. Let $\hat{\bx}$ be a solution to the correlation-based optimization problem~\eqref{eq:opt_corr}. Then, if
%\begin{align}\label{eq:sample_complexity_thm}
%  m &= \Omega\left(k \log\frac{Lr}{\delta} \cdot \left(1 + \right. \right. \\
% &\quad \left.\left. {} \left(\log^2 \left(k \log\frac{Lr}{\delta}\right)\right) \left(\log^2 n\right)\cdot \mathbbm{1}(\varsigma >0) \right) \right),
%\end{align}
\begin{equation}\label{eq:sample_complexity_thm}
 m = \Omega\left(k \log\frac{Lr}{\delta} \cdot \left(1 + \mathbbm{1}(\varsigma >0) S_{\varsigma} \right) \right),
\end{equation}
where $S_{\varsigma} = \big(\log^2 n\big) \cdot \big(\log^2 \big(k \log\frac{Lr}{\delta}\big)\big)$ is a sample complexity term concerning the adversarial noise, with probability at least $1-e^{-\Omega(k\log \frac{Lr}{\delta})} - e^{-\Omega(\sqrt{m})}-\mathbbm{1}(\varsigma >0) \cdot e^{-\Omega((\log^2 (k\log \frac{Lr}{\delta})) (\log^2 n))}$, it holds that
 \begin{equation}\label{eq:errbd_thm1}
  \|\hat{\bx} - \bx^*\|_2 \le O\left(\frac{\sqrt{k\log\frac{Lr}{\delta}}}{\lambda\sqrt{m}} +\frac{\varsigma}{\lambda}\right).
 \end{equation}
\end{theorem}
Note that the extra logarithmic terms $\log^2 (k\log\frac{Lr}{\delta})$ and $\log^2 n$ in~\eqref{eq:sample_complexity_thm} are only required for deriving an upper bound corresponding to the adversarial noise. If there is no adversarial noise, i.e., $\varsigma =0$, the sample complexity in~\eqref{eq:sample_complexity_thm} becomes $m = \Omega\big(k\log \frac{Lr}{\delta}\big)$, and it follows from~\eqref{eq:errbd_thm1} that $O\big(\frac{k \log \frac{Lr}{\delta}}{\lambda^2\epsilon^2}\big)$ samples suffice for an $\epsilon$-accurate recovery. This matches the scaling derived for i.i.d.~Gaussian measurements in~\cite{liu2020sample}. 

\begin{remark}\label{remark:v1}
As mentioned in the remark in~\cite[Page 8]{yu2017binary}, for ``typical'' unit vectors, we have $\|\bx^*\|_\infty = O\big(\frac{\log n}{\sqrt{n}}\big)$. Thus, assuming $m = \Theta\big(k \log \frac{Lr}{\delta}\big)$, we find that the condition $\rho = O\big(\frac{1}{m k \log \frac{Lr}{\delta}}\big)$ holds for typical signals when $\frac{1}{(k \log \frac{Lr}{\delta})^2} \gg \frac{\log n}{ \sqrt n }$.  Assuming $\frac{Lr}{\delta} = {\rm poly}(n)$~\cite{bora2017compressed}, it in turn suffices that $k = O(n^{\theta})$ for some $\theta \in \big(0,\frac{1}{4}\big)$. In addition, for an arbitrary unit vector $\bx^*$, we can use a simple trick to ensure the $\ell_\infty$ bound with high probability: Instead of using a randomly signed partial Gaussian circulant sensing matrix $\bA$, we can use $\bA\bH$ as the sensing matrix, where $\bH \in \bbR^{n \times n}$ is a randomly signed Hadamard matrix that also allows fast matrix-vector multiplications.  Then, the observation vector becomes $\bb = \theta(\bA\bH  \bx^*)$. Since $\bH  \bx^*$ satisfies the desired $\ell_\infty$ bound with high probability~\cite[Lemma 3.1]{vybiral2011variant}, for any $\hat{\bx}$ that solves~\eqref{eq:opt_corr}, $\|\hat{\bx}-\bH \bx^*\|_2$ satisfies the upper bound in~\eqref{eq:errbd_thm1}. Because $\bH$ is symmetric and orthogonal, we have the same upper bound for $\|\bH \hat{\bx} - \bx^*\|_2$.  
%Thus, if there is no adversarial noise, assuming $m = \Theta\big(k \log \frac{Lr}{\delta}\big)$, we find that the condition $\rho = O\big(\frac{1}{m k \log \frac{Lr}{\delta}}\big)$ holds for typical signals when $\frac{1}{(k \log \frac{Lr}{\delta})^2} \gg \frac{\log n}{ \sqrt n }$.  Assuming $\frac{Lr}{\delta} = {\rm poly}(n)$~\cite{bora2017compressed}, it in turn suffices that $k = O(n^{\theta})$ for some $\theta \in \big(0,\frac{1}{4}\big)$.
 %Therefore, the condition $\|\bx^*\|_\infty \le \rho = O\big(\frac{1}{m (t +\log m)}\big)$ implicitly forces $m = O\big(\frac{\sqrt{n}}{t+\log m}\big)$. In addition, as noted in~\cite{yu2017binary}, for a typical $n$-dimensional unit vector $\bx^*$, we have $\|\bx^*\|_\infty = O\big(\frac{\log n}{\sqrt{n}}\big)$. Thus, when $m = O\big(\frac{\sqrt{n}}{ (t +\log m) \log n}\big)$, the condition $\|\bx^*\|_\infty \le \rho$ may be satisfied for typical sequences. % Further, by using the idea from Ailon and Chazelle~\cite{ailon2006approximate}, we can pre-process the data by multiplying it with a randomly signed Hadamard matrix, and guarantee such an $\ell_\infty$ bound with high probability. 
\end{remark}

\begin{remark}\label{remark:v2}
 The use of random column sign flips and a randomly signed Hadamard matrix is most suited to applications where one has the freedom to choose any measurement matrix, but is limited by computation.  In this sense, the design is more practical than the widespread i.i.d.~Gaussian choice, though it may remain unsuitable in applications such as medical imaging where one {\em must} use subsampled Fourier measurements without sign flips or the Hadamard operator.
\end{remark}

\section{Proofs of Main Results}

In this section, we provide proofs for Lemma~\ref{lem:essential_circ} and Theorem~\ref{thm:corr_circulant}. We first present some useful auxiliary results that are general, and then some that are specific to our setup. 

\subsection{General Auxiliary Results}

%Before presenting the proof of Lemma~\ref{lem:essential_circ}, we state some useful auxiliary results. % First, we state the following standard definition.

\begin{definition} \label{def:subg}
 A random variable $X$ is said to be sub-Gaussian if there exists a positive constant $C$ such that $\left(\mathbb{E}\left[|X|^{p}\right]\right)^{1/p} \leq C  \sqrt{p}$ for all $p\geq 1$.  The sub-Gaussian norm of a sub-Gaussian random variable $X$ is defined as $\|X\|_{\psi_2}:=\sup_{p\ge 1} p^{-1/2}\left(\mathbb{E}\left[|X|^{p}\right]\right)^{1/p}$. 
\end{definition}

\iffalse
\begin{definition}
 A random variable $X$ is said to be sub-exponential if there exists a positive constant $C$ such that $\left(\bbE\left[|X|^p\right]\right)^{\frac{1}{p}} \le C p$ for all $p \ge 1$. The sub-exponential norm of $X$ is defined as
 \begin{equation}
  \|X\|_{\psi_1} = \sup_{p \ge 1} p^{-1} \left(\bbE\left[|X|^p\right]\right)^{\frac{1}{p}}.
 \end{equation}
\end{definition}
\fi

We have the following concentration inequality for sub-Gaussian random variables. 
\begin{lemma}{\em (Hoeffding-type inequality~\hspace{1sp}\cite[Proposition~5.10]{vershynin2010introduction})}
\label{lem:large_dev_Gaussian} Let $X_{1}, \ldots , X_{N}$ be independent zero-mean sub-Gaussian random variables, and let $K = \max_i \|X_i\|_{\psi_2}$. Then, for any $\balpha=[\alpha_1,\alpha_2,\ldots,\alpha_N]^T \in \mathbb{R}^N$ and any $t\ge 0$, it holds that
\begin{equation}
\mathbb{P}\left( \Big|\sum_{i=1}^{N} \alpha_i X_{i}\Big| \ge t\right) \le   \exp\left(1-\frac{ct^2}{K^2\|\balpha\|_2^2}\right).
\end{equation}
\end{lemma}

In particular, we have the following lemma concerning the Hoeffding's inequality for (real) Rademacher sums.
\begin{lemma}{\em (\hspace{1sp}\cite[Proposition~6.11]{rauhut2010compressive})}
\label{lem:bd_rademacher_real}
 Let $\bb \in \bbR^N$ and $\bepsilon = (\epsilon_j)_{j \in [N]}$ be a Rademacher sequence. Then, for $u >0$, 
 \begin{equation}\label{eq:bd_rademacher_real}
  \bbP\left(\left|\sum_{j=1}^N \epsilon_j b_j\right| \ge \|\bb\|_2 u\right) \le 2 \exp\left(-\frac{u^2}{2}\right).
 \end{equation}
\end{lemma}

We also use the following well-established fact based on the Riesz-Thorin interpolation theorem.
\begin{lemma}{\em(\hspace{1sp}\cite{bergh2012interpolation})}
\label{lem:riesz_thorin}
 For any matrix $\bA$, 
 \begin{equation}
  \|\bA\|_{2 \to 2} \le \max \left\{\|\bA\|_{1 \to 1}, \|\bA\|_{\infty \to \infty}\right\}.
 \end{equation}
\end{lemma}

\iffalse
In addition, we have the following concentration inequality for sums of independent sub-exponential random variables.
\begin{lemma}{\em (\hspace{1sp}\cite[Proposition~5.16]{vershynin2010introduction})}\label{lem:large_dev}
Let $X_{1}, \ldots , X_{N}$ be independent centered sub-exponential random variables, and $K = \max_{i} \|X_{i} \|_{\psi_{1}}$. Then for every $\balpha = [\alpha_1,\ldots,\alpha_N]^T \in \bbR^N$ and $\epsilon \geq 0$, it holds that
\begin{equation}
 \mathbb{P}\bigg( \Big|\sum_{i=1}^{N}\alpha_i X_{i}\Big|\ge \epsilon\bigg)  \leq 2  \exp \left(-c \cdot \mathrm{min}\Big(\frac{\epsilon^{2}}{K^{2}\|\balpha\|_2^2},\frac{\epsilon}{K\|\balpha\|_\infty}\Big)\right). \label{eq:subexp}
\end{equation}
%where $c>0$ is an absolute constant.
\end{lemma}
\fi

Next, we present the Hanson-Wright inequality, which gives a tail bound for quadratic forms in independent Gaussian variables.
\begin{lemma}{\em (Hanson-Wright inequality~\cite{hanson1971bound})}\label{lem:hanson_wright}
 If $\bg \in \bbR^N$ is a standard Gaussian vector and $\bB \in \bbR^{N\times N}$, then for any $t >0$, 
 \begin{align}
  \bbP\left(\left|\bg^T\bB\bg - \bbE\left[\bg^T\bB\bg\right]\right| \ge t\right) \le \exp\left(-c \cdot\min \left\{\frac{t^2}{\|\bB\|_\rmF^2},\frac{t}{\|\bB\|_{2\to 2}}\right\}\right).
 \end{align}
\end{lemma}

\subsection{Proof of Lemma~\ref{lem:essential_circ}}

%To prove Lemma~\ref{lem:essential_circ}, we use distinct proof techniques for different observation models. 
Recall that $\tilde{\bA} = \bR_{\Omega} \bC_{\bg}$ is a partial Gaussian circulant matrix without random column sign flips and $\tilde{\bx}^* = \bD_{\bxi} \bx^*$.  We have $\bA = \tilde{\bA} \bD_{\bxi}$, and $\ba_i = \bD_{\bxi} \tilde{\ba}_i$, where $\tilde{\ba}_i^T$ is the $i$-th row of $\tilde{\bA}$. Then,
\begin{align}
 b_i & = \theta(\langle \ba_i,\bx^*\rangle) = \theta(\langle \tilde{\ba}_i,\tilde{\bx}^*\rangle) \\
 & = \theta(\langle s_{\rightarrow i}(\bg),\tilde{\bx}^*\rangle) \\
 & = \theta(\langle \bg,s_{i \leftarrow}(\bD_{\bxi}\bx^*)\rangle).\label{eq:biEq}
\end{align}
Let $\bs_i = s_{i \leftarrow}(\bD_{\bxi}\bx^*)$. Since in general, the vectors $\bs_0, \bs_1,\ldots,\bs_{m-1}$ are not jointly orthogonal, the observations $b_0,b_1,\ldots,b_{m-1}$ are not independent. However, %for the cases of random flips with $\theta(x)=\tau \cdot \sign(x)$ and Gaussian noise with $\theta(x) = \sign(x+e)$, where $e\sim \calN(0,\sigma^2)$ ({\em cf.} Section~\ref{sec:setup}), 
the sequence of observations can be approximated by a sequence of independent random variables using the notion of $(m,\beta)$-orthogonality defined in the following. 

\begin{definition}
 A sequence of $m$ unit vectors $\bt_0,\bt_1,\ldots,\bt_{m-1}$ is said to be $(m,\beta)$-orthogonal if there exists a decomposition
 \begin{equation}\label{eq:siuiri}
  \bt_i = \bu_i + \br_i ~~(\forall i)
 \end{equation}
satisfying the following properties:
\begin{enumerate}
 \item $\bu_i$ is orthogonal to $\mathrm{span}\{\bu_j\,:\, j\ne i\}$;
 \item $\max_{i}\{\|\br_i\|_2\} \le \beta$.
\end{enumerate}
\end{definition}

We have the following useful lemma regarding $\bD_{\bxi}$.

\begin{lemma}{\em(\hspace{1sp}\cite[Lemma~7]{yu2017binary})}
\label{lem:almost_ortho}
 Let $\bp, \bq \in \bbR^n$ be vectors that satisfy $\|\bp\|_2 =1$ and $\|\bq\|_\infty \le \rho$ for some parameter $\rho$. Then, for any $0< j < n$ and $\epsilon >0$, we have 
 \begin{equation}
  \bbP\left(|\langle \bD_{\bxi} \bp,s_{ \rightarrow j}(\bD_{\bxi} \bq)\rangle| > \epsilon\right) \le e^{-\frac{\epsilon^2}{8\rho^2}},
 \end{equation}
where the probability is over the Rademacher sequence $\bxi$.
\end{lemma}

Based on Lemma~\ref{lem:almost_ortho}, and similar to~\cite[Lemma~9]{yu2017binary}, we have the following lemma ensuring that the $(m,\beta)$-orthogonality condition is satisfied by a certain sequence of unit vectors with high probability. The proof is given in Appendix~\ref{sec:proof_gamma_ortho} for completeness.
\begin{lemma}\label{lem:gamma_ortho}
 For any  $\nu \in (0,1)$, suppose that $\bx$ is a unit vector with $\|\bx\|_\infty \le \rho < \frac{1}{4\sqrt{2}m\log \frac{m^2}{\nu}}$. Let $\bt_i = s_{i \leftarrow }(\bD_{\bxi} \bx)$ for $i \in [m]$. Then, with probability at least $1-\nu$ over the choice of $\bxi$, the sequence $\bt_0,\bt_1,\ldots,\bt_{m-1}$ is $(m,\beta)$-orthogonal for $\beta = 2\sqrt{\rho}$.
\end{lemma}

In addition, we have the following two lemmas for the case of random noise added before quantization.
\begin{lemma}\label{lem:simple_gesum_eta}
 Let $g \sim \calN(0,1)$ and $e$ be any random noise that is independent with $g$. Then, for any $b \ge 0$ and $a, \eta >0$, 
 \begin{equation}
  \bbP(|ag+ b e| \le \eta) \le \sqrt{\frac{2}{\pi}} \frac{\eta}{a}.
 \end{equation}
\end{lemma}
\begin{proof}
 Let $f_e$ be the pdf of $e$ and $F_g$ be the cdf of $g$, i.e., $F_g(x) = \frac{1}{\sqrt{2\pi}}\int_{-\infty}^x e^{-\frac{t^2}{2}}\rmd t$. We have $\max_{x\in\bbR}F'_g(x) = \max_{x\in\bbR} \frac{1}{\sqrt{2\pi}} e^{-\frac{x^2}{2}} =  \frac{1}{\sqrt{2\pi}}$. Then,
 \begin{align}
  &\bbP(|ag+be| \le \eta)= \int_{-\infty}^{\infty} f_e(t) \int_{\frac{-\eta-bt}{a}}^{\frac{\eta -bt}{a}} \frac{1}{\sqrt{2\pi}} e^{-\frac{x^2}{2}}\rmd x \rmd t \\
  & = \int_{-\infty}^{\infty} f_e(t) \left(F_g\left(\frac{\eta -bt}{a}\right) - F_g\left(\frac{-\eta -bt}{a}\right)\right) \rmd t \\
  & \le \frac{2 \eta}{a}\cdot \max_{x\in\bbR}F'_g(x)  \int_{-\infty}^{\infty} f_e(t) \rmd t \\
  & = \sqrt{\frac{2}{\pi}} \frac{\eta}{a}.
 \end{align}
\end{proof}
\iffalse
Moreover, we have the following lemma.% which says that if $F_e(x) + F_e(-x) =1$ for any $x$, then for any $a>0$, $\sign(g+\frac{e}{a})$ is a Rademacher random variable.
\begin{lemma}\label{lem:Rademacher_ge}
 If $F_e(x) + F_e(-x) =1$ for any $x \in \bbR$, then for any $a>0$, $X_a:= \sign(g+\frac{e}{a})$ is a Rademacher random variable, i.e., $\bbP(X_a = 1) = \bbP(X_a=-1) =\frac{1}{2}$.
\end{lemma}
\begin{proof}
 Let $I_a = \bbP(X_a < 0)$. It suffices to show that $I_a = \frac{1}{2}$. We have 
 \begin{align}
  I_a  &= \frac{1}{\sqrt{2\pi}} \int_{-\infty}^\infty e^{-\frac{x^2}{2}} \int_{-\infty}^{-ax} f_e(t)\rmd t  \rmd x \\
  & = \frac{1}{\sqrt{2\pi}} \int_{-\infty}^\infty e^{-\frac{x^2}{2}} F_e(-ax)  \rmd x \\
  & = \frac{1}{\sqrt{2\pi}} \int_{-\infty}^\infty e^{-\frac{x^2}{2}} (1-F_e(ax))  \rmd x  \\
  & = 1 - \frac{1}{\sqrt{2\pi}} \int_{-\infty}^\infty e^{-\frac{x^2}{2}} F_e(ax)  \rmd x \\
  & \overset{x = -y}{=} 1 - \frac{1}{\sqrt{2\pi}} \int_{-\infty}^\infty e^{-\frac{y^2}{2}} F_e(-ay)  \rmd y \\
  & = 1- I_a,
 \end{align}
which gives the desired result.
\end{proof}
\fi
The following lemma shows that when the random noise $e$ added before quantization is Gaussian or logit, then the additional assumptions corresponding to $e$ made in Section~\ref{sec:setup} are satisfied. 
\begin{lemma}\label{lem:add_assumps_e}
 When $e \sim \calN(0,\sigma^2)$ or $e$ is a logit noise term, for any $x \in \bbR$ and any $a >0$, we have 
 \begin{equation}\label{eq:FeSymm}
  F_e(x) + F_e(-x) =1 
 \end{equation}
and 
\begin{equation}
 \lambda_a := \bbE[\theta_a(g)g] < C_1 a +C_2,
\end{equation}
where $\theta_a(x) :=\sign(x+\frac{e}{a})$, and $C_1,C_2$ are absolute constants. 
\end{lemma}
\begin{proof}
 When $e \sim \calN(0,\sigma^2)$, by the symmetry of the pdf of a zero-mean Gaussian distribution, it is easy to see that~\eqref{eq:FeSymm} holds. Note that $\frac{e}{a} \sim \calN(0,\frac{\sigma^2}{a^2})$. From $\lambda = \bbE[\theta(g)g] = \sqrt{\frac{2}{\pi(1+\sigma^2)}}$ ({\em cf.} Section~\ref{sec:setup}), we obtain
 \begin{equation}
  \lambda_a = \bbE[\theta_a(g)g] = \sqrt{\frac{2}{\pi(1+\frac{\sigma^2}{a^2})}} < \sqrt{\frac{2}{\pi}}.
 \end{equation}
When $e$ is a logit noise term, we have 
\begin{equation}
 \bbP(\theta(x) = 1) = \frac{e^x}{e^x +1}, \quad \bbP(\theta(x) = -1) = \frac{1}{e^x +1}.
\end{equation}
From $\theta(x) =\sign(x+e)$, we obtain
\begin{equation}
 F_e(-x) = \bbP(e < -x) = \frac{1}{e^x + 1},
\end{equation}
or equivalently,
\begin{equation}
 F_e(x) = \frac{1}{e^{-x} + 1}.
\end{equation}
This gives $F_e(x) + F_e(-x) = 1$ and 
\begin{equation}
 f_e(x) = F'_e(x)= \frac{e^{-x}}{(e^{-x} + 1)^2} \le\frac{1}{4}. 
\end{equation}
For any fixed $x$, we have 
\begin{equation}
 \phi_a(x) := \bbE[\theta_a(x)] = 1-2\bbP(e < -ax).
\end{equation}
Then, for $g \sim \calN(0,1)$, we obtain 
\begin{align}
 &\lambda_a := \bbE[\theta_a(g)g] = \bbE[\bbE[\theta_a(g)g|g]] \\
 & = \bbE[\phi_a(g)g] =\bbE[\phi'_a(g)] = 2a \bbE[f_e(-ag)] \le \frac{a}{2}.
\end{align}
\end{proof}

\textbf{Overview of the proof}: With the above auxiliary results in place, the main ideas for proving Lemma~\ref{lem:essential_circ} are outlined as follows: Let $\bs_i = s_{i \leftarrow}(\bD_{\bxi}\bx^*)$. From Lemma~\ref{lem:gamma_ortho}, the sequence of unit vectors $\bs_0,\bs_1,\ldots,\bs_{m-1}$ is $(m,\beta)$-orthogonal with high probability. We wirte $\bs_i = \bu_i + \br_i$ as in~\eqref{eq:siuiri}, and then $b_i = \theta(\langle \bg,\bs_i\rangle) =  \theta(\langle \bg,\bu_i\rangle + \langle \bg,\br_i\rangle)$. Because $\|\br_i\|_2 = o(1)$, for the case of random bit flips with $\theta(x) = \tau\cdot \sign(x)$, we have that for most of $i \in [m]$, the term $|\langle \bg,\bu_i\rangle|$ is larger than the term $|\langle \bg,\br_i\rangle|$, and thus $b_i =  \theta(\langle \bg,\bu_i\rangle + \langle \bg,\br_i\rangle) = \theta(\langle \bg,\bu_i\rangle)$.\footnote{Similarly, for random noise added before quantization with $\theta(x) =\sign(x+e)$, from Lemma~\ref{lem:simple_gesum_eta}, we have that for most of $i \in [m]$, the term $|\langle \bg,\bu_i\rangle + e_i|$ is larger than the term $|\langle \bg,\br_i\rangle|$, and thus we also have $b_i =\theta(\langle \bg,\bu_i\rangle)$.} Because of the orthogonality of $\bu_0,\bu_1,\ldots,\bu_{m-1}$, the sequence $\{\theta(\langle \bg,\bu_i\rangle)\}_{i \in [m]}$ is independent. Therefore, for any $j \in [n]$, the $j$-th entry of $\frac{1}{m}\tilde{\bA}^T \bb -\lambda \tilde{\bx}^*$ is
\begin{equation}
 \frac{1}{m}\sum_{i \in [m]} (\tilde{a}_{ij}b_i - \lambda \tilde{x}_j^*) \approx \frac{1}{m}\sum_{i \in [m]} (g_{j-i} \theta(\langle \bg,\bu_i\rangle) - \lambda \bu_i(j-i)),
\end{equation}
where we use $\tilde{a}_{ij} = g_{j-i}$, $b_i \approx \theta(\langle \bg,\bu_i\rangle)$ and $\tilde{x}_j^* = \bs_i(j-i) \approx \bu_i(j-i)$. The sequence $\big\{g_{j-i} \theta(\langle \bg,\bu_i\rangle) - \lambda \bu_i(j-i)\big\}_{i \in [m]}$ is still not independent, but defining $h_i = \big\langle \bg,\frac{\bu_i}{\|\bu_i\|_2}\big\rangle$, we have that $h_0,h_1,\ldots,h_{m-1}$ are i.i.d.~standard Gaussian random variables. We use these to represent $g_{j-i}$, such that the sum $\frac{1}{m}\sum_{i \in [m]} (g_{j-i} \theta(\langle \bg,\bu_i\rangle) - \lambda \bu_i(j-i))$ can be decomposed into three terms.\footnote{For random noise added before quantization, it is decomposed into four terms.} We control these three terms separately by making use of Lemmas~\ref{lem:large_dev_Gaussian},~\ref{lem:bd_rademacher_real} and~\ref{lem:hanson_wright}. We now proceed with the formal proof. A table of notations is presented in Appendix~\ref{sec:table_proof_lemma1} for convenience.
\begin{proof}[Proof of Lemma~\ref{lem:essential_circ}]
\textbf{Useful high-probability events}:
  Recall that from~\eqref{eq:biEq}, we have $b_i = \theta(\langle \bg,s_{i \leftarrow}(\bD_{\bxi}\bx^*)\rangle)$. Let $\bs_i = s_{i \leftarrow}(\bD_{\bxi}\bx^*)$ for $i \in [m]$. Because $\|\bx^*\|_\infty \le \rho$ with $\rho \le O\big(\frac{1}{m (t+\log m)}\big)$, setting $\nu = e^{-t}$ in~Lemma~\ref{lem:gamma_ortho}, we have that with probability at least $1-e^{-t}$, the sequence $\bs_0,\ldots,\bs_{m-1}$ is $(m,\beta)$-orthogonal with $\beta= 2\sqrt{\rho}$. We write $\bs_i = \bu_i + \br_i$ with $\|\br_i\|_2 \le \beta$ for all $i$. In the following, we always assume the $(m,\beta)$-orthogonality. Because $\|\br_i\|_2 \le \beta$, from Lemma~\ref{lem:large_dev_Gaussian} and a union bound over $i \in [m]$, we have that for any $\eta >0$, with probability at least $1-m e^{-\frac{\eta^2}{2\beta^2}}$,
\begin{equation}\label{eq:upper_bd_ris}
 \max_{i \in [m]}|\langle \bg,\br_i\rangle| \le  \eta.
\end{equation} 
In addition, for $\beta = o(1)$, we have $\|\bu_i\|_2 \in [1-\beta,1+\beta] \ge \sqrt{\frac{2}{\pi}}$. By Lemma~\ref{lem:simple_gesum_eta}, we obtain for any $i\in [m]$ and $\eta >0$ that
\begin{align}
 \bbP\left(\left|\left\langle \bg,\bu_i\right\rangle + \omega e_i\right| \le \eta \right) &=  \bbP\left(\left|\|\bu_i\|_2\left\langle \bg,\frac{\bu_i}{\|\bu_i\|_2}\right\rangle + \omega e_i\right| \le \eta \right)\\
 & \le \sqrt{\frac{2}{\pi}} \frac{\eta}{\|\bu_i\|_2} \le \eta, \label{eq:inner_product_bound}
\end{align}
where $\omega =0$ corresponds to random bit flips with $\theta(x) = \tau \cdot \sign(x)$, and $\omega =1$ corresponds to random noise added before quantization with $\theta(x) = \sign(x + e)$. By the orthogonality of $\bu_0,\ldots,\bu_{m-1}$, the sequence $\langle \bg,\bu_0\rangle,\ldots,\langle \bg,\bu_{m-1}\rangle$ is independent. Let 
\begin{equation}
 I_{\eta} = \{i \in [m]\,:\, |\langle \bg,\bu_i\rangle + \omega e_i| \le \eta \},
\end{equation}
and let $I_{\eta}^c = [m] \backslash I_{\eta}$ be the complement of $I_\eta$. Using a standard Chernoff bound~\cite{yu2017binary}, we have for any $\zeta > 0$ that
\begin{equation}\label{eq:I_eta_bd}
 \bbP(|I_{\eta}| \ge m \eta + m \zeta) < e^{-\frac{m \zeta^2}{2\eta +\zeta}}.
\end{equation}
Note that when~\eqref{eq:upper_bd_ris} holds, for $i \in I_\eta^c$, we have 
\begin{equation}\label{eq:bi_Jeta}
 b_i = \theta(\langle \bg,\bs_i \rangle) = \theta(\langle \bg,\bu_i \rangle),
\end{equation}
since the term $\langle \bg,\bu_i\rangle + \omega e_i$ of magnitude exceeding $\eta$ dominates the other term $\langle \bg,\br_i\rangle$ of magnitude at most $\eta$.  In addition, because $\bs_i = s_{i \leftarrow }(\tilde{\bx}^*)$, we have for any fixed $j \in [n]$ that
\begin{equation}\label{eq:siJmI}
 \bs_i (j-i) = \tilde{x}^*_j. 
\end{equation}
Another useful property is the following bound on the maximum of $n$ independent Gaussians (e.g., using \cite[Eq.~(2.9)]{wainwright2019high} and a union bound): With probability at least $1-e^{-t}$,
\begin{equation}\label{eq:upper_bd_gis}
\max_{j \in [n]} |g_j| \le \sqrt{2(t+\log n)}. 
\end{equation}
Combining \eqref{eq:upper_bd_ris}, \eqref{eq:I_eta_bd}, and~\eqref{eq:upper_bd_gis}, we find that the event
\begin{align}
 \mathcal{E} \,:\, &   \max_{i \in [m]}|\langle \bg,\br_i\rangle| \le \eta; \quad |I_\eta| <  m (\eta+\zeta); \quad \max_{j \in [n]}|g_j| \le \sqrt{2(t+\log n)}; \label{eq:event_e}
\end{align}
occurs with probability at least $1-e^{-t} - m e^{-\frac{\eta^2}{2\beta^2}}- e^{-\frac{m \zeta^2}{2\eta +\zeta}}$. 

\textbf{Dividing $\frac{1}{m}\tilde{\bA}^T \bb - \lambda \tilde{\bx}^*$ into several terms}:
By the definitions of $I_\eta$ and $I_\eta^c$, for any $j \in [n]$, we have 
\begin{align}
 & \left[\frac{1}{m}\tilde{\bA}^T \bb - \lambda \tilde{\bx}^*\right]_j = \frac{1}{m}\sum_{i=0}^{m-1} (\tilde{a}_{ij} b_i - \lambda \tilde{x}_j^*)  \\
 & = \frac{1}{m}\sum_{i \in I_\eta^c} (\tilde{a}_{ij} b_i - \lambda \tilde{x}_j^*) + \frac{1}{m}\sum_{i \in I_\eta} (\tilde{a}_{ij} b_i - \lambda \tilde{x}_j^*) \\
 & = \frac{1}{m}\sum_{i \in I_\eta^c} (\tilde{a}_{ij} \theta(\langle \bg,\bu_i\rangle) - \lambda \bs_i (j-i))  + \frac{1}{m}\sum_{i \in I_\eta} (\tilde{a}_{ij} b_i - \lambda \tilde{x}_j^*) 
\label{eq:decomp_sumSecond}\\
 & = \frac{1}{m}\sum_{i \in I_\eta^c} \left(g_{j-i} \theta(\langle \bg,\bu_i\rangle) - \lambda \frac{\bu_i (j-i)}{\|\bu_i\|_2}\right) \nonumber \\
 & \indent - \frac{\lambda}{m}\sum_{i \in I_\eta^c} \left(\br_i (j-i) + \bu_i(j-i) - \frac{\bu_i (j-i)}{\|\bu_i\|_2}\right) \nonumber \\
 & \indent + \frac{1}{m}\sum_{i \in I_\eta} (g_{j-i} b_i - \lambda \tilde{x}_j^*), \label{eq:three_terms}
\end{align}
where~\eqref{eq:decomp_sumSecond} is from~\eqref{eq:bi_Jeta} and~\eqref{eq:siJmI}, and~\eqref{eq:three_terms} is from $\bs_i =\bu_i + \br_i$. Therefore, if we denote $Y_{ij} = g_{j-i} \theta(\langle\bg,\bu_i\rangle) - \lambda\frac{\bu_i (j-i)}{\|\bu_i\|_2}$, $Z_1 := 2(\eta + \zeta)\sqrt{2(t+\log n)}$, $Z_2:= C\sqrt{\frac{t+\log n}{m}}$ for some sufficiently large absolute constant $C$, and let $Z_3 :=  Z_2 + 2\beta + 2Z_1$, we obtain
\begin{align}
 & \bbP\left(\max_{j \in [n]} \left|\frac{1}{m}\sum_{i=0}^{m-1} (\tilde{a}_{ij} b_i - \lambda \tilde{x}_j^*)\right| >   Z_3\right) \le \bbP(\calE^c) + \bbP\left(\max_{j \in [n]} \left|\frac{1}{m}\sum_{i=0}^{m-1} (\tilde{a}_{ij} b_i - \lambda \tilde{x}_j^*)\right| >  Z_3, \calE\right) \label{eq:begin_Z123}\\
 & \le e^{-t} + m e^{-\frac{\eta^2}{2\beta^2}}+ e^{-\frac{m \zeta^2}{2\eta +\zeta}} + \bbP\left(\max_{j \in [n]} \left|\frac{1}{m}\sum_{i=0}^{m-1} (\tilde{a}_{ij} b_i - \lambda \tilde{x}_j^*)\right| >  Z_3, \calE\right) \label{eq:complex_bd_1}\\
 & \le e^{-t} + m e^{-\frac{\eta^2}{2\beta^2}}+ e^{-\frac{m \zeta^2}{2\eta +\zeta}}  + \bbP\left(\max_{j \in [n]}\left|\frac{1}{m}\sum_{i \in I_\eta^c} Y_{ij}\right| >Z_2 +  Z_1, \calE\right) \nonumber\\ 
 & \indent + \bbP\left(\max_{j \in [n]} \left|\frac{\lambda}{m}\sum_{i \in I_\eta^c} \left(\bs_i (j-i) - \frac{\bu_i (j-i)}{\|\bu_i\|_2}\right|\right) > 2\beta, \calE\right) + \bbP\left( \max_{j \in [n]} \left|\frac{1}{m}\sum_{i \in I_\eta} (g_{j-i} b_i - \lambda \tilde{x}_j^*)\right| > Z_1, \calE\right) \label{eq:complex_bd_2} \\
 & \le e^{-t} + m e^{-\frac{\eta^2}{2\beta^2}}+ e^{-\frac{m \zeta^2}{2\eta +\zeta}}  + \bbP\left(\max_{j \in [n]}\left|\frac{1}{m}\sum_{i \in I_\eta} Y_{ij}\right| >  Z_1 \Big| \calE\right) + \bbP\left(\max_{j \in [n]}\left|\frac{1}{m}\sum_{i \in [m]} Y_{ij}\right| >Z_2\right) \nonumber\\
 & \indent+ \bbP\left(\max_{j \in [n]} \left|\frac{\lambda}{m}\sum_{i \in I_\eta^c} \left(\bs_i (j-i)  - \frac{\bu_i (j-i)}{\|\bu_i\|_2}\right)\right| > 2\beta \Big| \calE\right) + \bbP\left( \max_{j \in [n]} \left|\frac{1}{m}\sum_{i \in I_\eta} (\tilde{a}_{ij} b_i - \lambda \tilde{x}_j^*)\right| > Z_1\Big| \calE\right), 
 \label{eq:complex_four_terms}
\end{align}
where~\eqref{eq:complex_bd_1} follows from~\eqref{eq:event_e},~\eqref{eq:complex_bd_2} follows from~\eqref{eq:three_terms} and the definition of $Z_3$, and~\eqref{eq:complex_four_terms} uses the triangle inequality and the fact that $\bbP(\calA,\calE) \le \min\{ \bbP(\calA|\calE), \bbP(\calA) \}$. Noting that $\tilde{a}_{ij} = g_{j-i}$, $\lambda <1$ ({\em cf.} Section~\ref{sec:setup}), $|\tilde{x}_j^*|\le \rho <1$, $\|\br_i\|_2 \le \beta$, and $\left|1-\|\bu_i\|_2\right| \le \beta$, conditioned on $\calE$, we have
\begin{equation}
 \max_{j \in [n]} \left|\frac{\lambda}{m}\sum_{i \in I_\eta^c} \left(\br_i (j-i) + \bu_i(j-i) - \frac{\bu_i (j-i)}{\|\bu_i\|_2}\right)\right| \le 2\beta,\label{eq:complex_four_terms1}
\end{equation}
\begin{equation}
 \max_{j \in [n]} \left|\frac{1}{m}\sum_{i \in I_\eta} (g_{j-i} b_i - \lambda \tilde{x}_j^*)\right| \le Z_1, \label{eq:complex_four_terms2}
\end{equation}
and
\begin{align}
 \max_{j \in [n]}\left| \frac{1}{m}\sum_{i \in I_\eta} \left(g_{j-i} \theta(\langle \bg,\bu_i\rangle) - \lambda \frac{\bu_i (j-i)}{\|\bu_i\|_2}\right)\right| \le Z_1. \label{eq:complex_four_terms3}
\end{align}

\textbf{Bounding the remaining term}: In the following, we focus on deriving an upper bound for $\max_{j \in [n]}\big|\frac{1}{m}\sum_{i \in [m]} Y_{ij}\big| = \max_{j \in [n]}\big|\frac{1}{m}\sum_{i \in [m]} \big(g_{j-i} \theta(\langle \bg,\bu_i\rangle) - \lambda \frac{\bu_i (j-i)}{\|\bu_i\|_2}\big)\big|$. Let $h_i = \big\langle \bg,\frac{\bu_i}{\|\bu_i\|_2}\big\rangle$. Then, $h_0,h_1,\ldots,h_{m-1}$ are i.i.d. standard Gaussian random variables. For $i, \ell \in [m]$, $j \in [n]$, we denote $\bar{u}_{j,i,\ell} := \frac{\bu_\ell(j-i)}{\|\bu_\ell\|_2}$. We have from the definition of $h_\ell$ that $\mathrm{Cov}[g_{j-i},h_\ell] = \bar{u}_{j,i,\ell}$, and for any fixed $j \in [n]$, $g_{j-i}$ can be written as 
\begin{equation}\label{eq:g_jminusi}
 g_{j-i} = \sum_{\ell \in [m]} \bar{u}_{j,i,\ell} h_\ell + w_{j,i} t_{j,i},
\end{equation}
where $w_{j,i} = \sqrt{1- \sum_{\ell \in [m]} \bar{u}_{j,i,\ell}^2}$ and $t_{j,i} \sim \calN(0,1)$ is independent of $h_0,h_1,\ldots,h_{m-1}$. Note that for random noise added before quantization, we have 
\begin{align}
 &\theta(\langle\bu_i,\bg\rangle) = \theta(\|\bu_i\|_2 h_i) = \sign(\|\bu_i\|_2 h_i + e_i ) \\
 & =  \mathrm{sign}\left( h_i +\frac{ e_i}{\|\bu_i\|_2}\right) = \theta_i(h_i), \label{eq:theta_i}
\end{align}
where $\theta_i(x):=\sign(x + \frac{e}{\|\bu_i\|_2})$. Let $\lambda_i = \bbE[\theta_i(g)g]$ for $g \sim\calN(0,1)$. From Lemma~\ref{lem:add_assumps_e}, we have that $\lambda_i < C_1 \|\bu_i\|_2 + C_2 =C$, and $\theta_i(g)$ is a Rademacher random variable. Therefore, we obtain\footnote{For the case of random bit flips with $\theta(x) = \tau \cdot\sign(x+e)$, we set $\theta_i = \theta$ (and thus $\lambda_i = \lambda$) for all $i \in [m]$. }
\begin{align}
 & \frac{1}{m}\sum_{i \in [m]} \left(g_{j-i} \theta(\langle \bg,\bu_i\rangle) - \lambda \frac{\bu_i (j-i)}{\|\bu_i\|_2}\right) = \frac{1}{m}\sum_{i \in [m]} \left(g_{j-i} \theta_i(h_i) - \lambda \bar{u}_{j,i,i}\right) \label{eq:minorDiff}\\
 & = \frac{1}{m}\sum_{i \in [m]} \left(\sum_{\ell \in [m]} \bar{u}_{j,i,\ell} h_\ell \theta_i(h_i)  + w_{j,i} t_{j,i} \theta_i(h_i) - \lambda \bar{u}_{j,i,i}\right) \label{eq:last_sumThree} \\
 & = \frac{1}{m}\sum_{i \in [m]} \bar{u}_{j,i,i}\left( h_i \theta_i(h_i) - \lambda_i \right) +\frac{1}{m}\sum_{i\in [m]} \bar{u}_{j,i,i}(\lambda_i -\lambda)   \nonumber \\
 & \indent + \frac{1}{m}\sum_{i \in [m]} \theta_i(h_i) \sum_{\ell \ne i} \bar{u}_{j,i,\ell} h_\ell + \frac{1}{m}\sum_{i \in [m]} w_{j,i} t_{j,i} \theta_i(h_i) ,\label{eq:last_decompThree}
\end{align}
where~\eqref{eq:last_sumThree} follows from~\eqref{eq:g_jminusi}. We control the four summands in~\eqref{eq:last_decompThree} separately as follows:
\begin{itemize}
 \item Because $\bar{u}_{j,i,i} h_i \theta_i(h_i) - \lambda_i \bar{u}_{j,i,i}$, $i \in [m]$, are independent, zero-mean, sub-Gaussian random variables with sub-Gaussian norm upper bounded by an absolute constant $c$, from Lemma~\ref{lem:large_dev_Gaussian} and a union bound over $j \in [n]$, we have with probability at least $1- e^{-t}$ that
 \begin{equation}\label{eq:last_decompThree1}
  \max_{j \in [n]}\left|\frac{1}{m}\sum_{i \in [m]} \bar{u}_{j,i,i}\left( h_i \theta_i(h_i) - \lambda_i \right)\right|  \le O\left(\sqrt{\frac{t+\log n}{m}}\right). 
 \end{equation}
 
 \item For the case of random bit flips, the second term vanishes. For the case of random noise added before quantization, from $0<\lambda,\lambda_i < C$ and~\eqref{eq:maxU5}, we obtain
 \begin{align}
  &\max_{j \in [n]}\left|\frac{1}{m}\sum_{i\in [m]} \bar{u}_{j,i,i}(\lambda_i -\lambda)\right| \le C \max_{i \in [m],j\in[n]} |\bar{u}_{j,i,i}| = O\left(\frac{1}{\sqrt{m(t+\log m)}}\right). \label{eq:last_decompThree1lambda}
 \end{align}
 
 \item We have
 \begin{align}
  &\max_{j,\ell} \sum_{i \in [m]} \bar{u}_{j,i,\ell}^2 \le 4 \max_{j,\ell} \sum_{i \in [m]} \bu_{\ell} (j - i)^2 \label{eq:maxU1} \\
  & = 4 \max_{j,\ell} \sum_{i \in [m]}(\bs_{\ell}(j-i) - \br_{\ell}(j-i))^2 \label{eq:maxU2} \\
  & \le  8 \max_{j,\ell} \sum_{i \in [m]} \left(\bs_{\ell}(j-i)^2 +\br_{\ell}(j-i)^2\right) \label{eq:maxU3}\\ 
  & \le 8 \max_{j,\ell} \left(m \|\bx^*\|_\infty^2 + \|\br_\ell\|_2^2\right)\label{eq:maxU4}\\
  & = O\left(\frac{1}{m(t + \log m)}\right),\label{eq:maxU5}
 \end{align}
where we use $\|\bu_\ell\|_2 = 1 + o(1)$ in~\eqref{eq:maxU1}, $\bs_\ell = \bu_\ell + \br_\ell$ in~\eqref{eq:maxU2}, $(a-b)^2 \le 2(a^2 + b^2)$ in~\eqref{eq:maxU3}, and $\|\bs_\ell\|_\infty = \|\bx^*\|_\infty \le \rho = O\big(\frac{1}{m (t+\log m)}\big)$ as well as $\|\br_\ell\|_2 \le 2\sqrt{\rho}$  in~\eqref{eq:maxU4} and~\eqref{eq:maxU5}. Then, from Proposition~\ref{prop:prop1} below and a union bound over $j \in [n]$, we have with probability at least $1-O\big(e^{-t}\big)$ that 
\begin{align}\label{eq:last_decompThree2}
 &\max_{j \in [n]}\left|\frac{1}{m}\sum_{\ell \in [m]} h_\ell\sum_{i \ne \ell} \bar{u}_{j,i,\ell} \theta_i(h_i)\right| \le O\left(\sqrt{\frac{t+\log n + \log m }{m}}\right) = O\left(\sqrt{\frac{t+\log n}{m}}\right). 
\end{align}

\item For a fixed $j \in [n]$ and $0 \le i_1 \ne i_2 <m$, we have 
\begin{align}
 \mathrm{Cov}\left[w_{j,i_1} t_{j,i_1},w_{j,i_2}t_{j,i_2}\right] &= \mathrm{Cov}\left[g_{j-i_1} - \sum_{\ell \in [m]} \bar{u}_{j,i_1,\ell} h_\ell, g_{j-i_2} - \sum_{\ell \in [m]} \bar{u}_{j,i_2,\ell} h_\ell\right]\\
 & = -\sum_{\ell\in [m]} \bar{u}_{j,i_1,\ell}\bar{u}_{j,i_2,\ell}.
\end{align}
Then, $[w_{j,0}t_{j,0},\ldots,w_{j,m-1}t_{j,m-1}]^T \in \bbR^m$ is a Gaussian vector with zero mean and covariance matrix $\bSigma_j$, where $\Sigma_{j,i,i} = w_{j,i}^2 = 1- \sum_{\ell \in [m]} \bar{u}_{j,i,\ell}^2 \le 1$ and from~\eqref{eq:maxU4}, we have for $i_1 \ne i_2$ that $\big|\Sigma_{j,i_1,i_2}\big| = \big|-\sum_{\ell\in [m]} \bar{u}_{j,i_1,\ell}\bar{u}_{j,i_2,\ell}\big| \le \sqrt{(\sum_{\ell \in [m]} \bar{u}_{j,i_1,\ell}^2) (\sum_{\ell \in [m]} \bar{u}_{j,i_2,\ell}^2)} = O\big(\frac{1}{m(t+\log m)}\big) = O\big(\frac{1}{m}\big)$.  When $m = \Omega(t+\log n)$, taking a union bound over $j \in [n]$, and setting $\epsilon = O\big(\sqrt{\frac{t+\log n}{m}}\big)$ in Proposition~\ref{prop:prop2} below, we obtain with probability at least $1-O\big(e^{-t}\big)$ that
 \begin{equation}\label{eq:last_decompThree3}
  \max_{j \in [n]}\left|\frac{1}{m}\sum_{i \in [m]} w_{j,i} t_{j,i} \theta_i(h_i) \right| \le O\left(\sqrt{\frac{t + \log n}{m}}\right). 
 \end{equation}
\end{itemize}

Combining~\eqref{eq:last_decompThree} with~\eqref{eq:last_decompThree1},~\eqref{eq:last_decompThree1lambda},~\eqref{eq:last_decompThree2},~\eqref{eq:last_decompThree3}, we obtain with probability at least $1-O\left(e^{-t}\right)$ that 
\begin{align}
 &\max_{j \in [n]}\left|\frac{1}{m}\sum_{i \in [m]} \left(g_{j-i} \theta(\langle \bg,\bu_i\rangle) - \lambda \frac{\bu_i (j-i)}{\|\bu_i\|_2}\right)\right| \le O\left(\sqrt{\frac{t+\log n}{m}}\right). \label{eq:complex_four_terms4}
\end{align}

\textbf{Combining terms}: Finally, combining the $(m,\beta)$-orthogonality and~\eqref{eq:complex_four_terms},~\eqref{eq:complex_four_terms1},~\eqref{eq:complex_four_terms2},~\eqref{eq:complex_four_terms3}, and~\eqref{eq:complex_four_terms4}, with probability at least $1- O\left(e^{-t}\right) - m e^{-\frac{\eta^2}{2\beta^2}} - e^{-\frac{m \zeta^2}{2\eta +\zeta}}$, we have 
\begin{align}
 &\left\|\frac{1}{m}\tilde{\bA}^T \bb  -\lambda \tilde{\bx}^*\right\|_\infty = \max_{j \in [n]} \left|\frac{1}{m}\sum_{i=0}^{m-1} (\tilde{a}_{ij} b_i - \lambda \tilde{x}_j^*)\right| \le  O\left(\sqrt{\frac{t+\log n}{m}}\right) + 2\beta + 4(\eta + \zeta)\sqrt{2(t+\log n)}.
\end{align}
Setting $\eta = \beta \sqrt{2(t+\log m)} = O(\sqrt{\rho})\cdot \sqrt{2(t+\log m)} = O\big(\frac{1}{\sqrt{m}}\big)$ and $\zeta = \frac{C}{\sqrt{m}}$, we obtain with probability at least $1- e^{-\Omega(t)} - e^{-\Omega(\sqrt{m})}$ that 
\begin{equation}
 \left\|\frac{1}{m}\tilde{\bA}^T \bb  -\lambda \tilde{\bx}^*\right\|_\infty \le  O\left(\sqrt{\frac{t+\log n}{m}}\right).
\end{equation}
\end{proof}

We now move on to the statements and  proofs of Propositions~\ref{prop:prop1} and~\ref{prop:prop2} used above. Throughout the following, we use the $\theta_i$ defined in~\eqref{eq:theta_i}.\footnote{For the case of random bit flips with $\theta(x) = \tau \cdot\sign(x+e)$, we set $\theta_i = \theta$ for $i \in [m]$.}
%we assume $\theta(x) = \tau \cdot \sign(x)$ (random flips) or $\theta(x) =\sign(x+e)$ with $e \sim \calN(0,\sigma^2)$ (Gaussian noise). 
\begin{proposition}\label{prop:prop1}
 Let $\bz \in \bbR^m$ be a standard Gaussian vector. For any $t>0$, suppose that $\bW \in \bbR^{m \times m}$ satisfies $\max_{\ell \in [m]} \sum_{i \in [m]} w_{i,\ell}^2  = O\big(\frac{1}{m(t+\log m)}\big)$. Then, we have with probability at least $1-O\big(e^{-t}\big)$ that
 \begin{equation}
  \left|\frac{1}{m}\sum_{i \in [m]} \theta_i(z_i) \sum_{\ell \ne i} w_{i,\ell} z_\ell\right| \le O\left(\sqrt{\frac{t+\log m}{m}}\right).
 \end{equation}
\end{proposition}
\begin{proof}
For any fixed $\ell \in [m]$ and any $v > 0$, from Lemma~\ref{lem:large_dev_Gaussian}, we have with probability at least $1-O\big(e^{-v^2}\big)$ that 
\begin{equation}\label{eq:prop1_eq1}
 |z_\ell| \le v.
\end{equation}
Therefore, with probability at least $1-O\big(e^{-v^2}\big)$,
\begin{equation}\label{eq:prop1_eq2}
 |z_\ell| \sqrt{\sum_{i \ne \ell} w_{i,\ell}^2} = O\left(\frac{v}{\sqrt{m(t+\log m)}}\right).
\end{equation}
Note that $\theta_i(z_0),\theta_i(z_1),\ldots,\theta_i(z_{m-1})$ form a Rademacher sequence. Then, for any $u > 0$ and a sufficiently large absolute constant $C$, we have 
\begin{align}
 &\bbP\left(\left|z_\ell \sum_{i \ne \ell} w_{i,\ell} \theta_i(z_i)\right| >\frac{Cv}{\sqrt{m(t+\log m)}} u\right) \nonumber \\
 & \le  \bbP(|z_\ell| >v)  +  \bbP\left(\left|z_\ell \sum_{i \ne \ell} w_{i,\ell} \theta_i(z_i)\right| >\frac{Cv}{\sqrt{m(t+\log m)}} u\text{ } \Big| \text{ } |z_\ell| \le v\right)\\
 & \le O\big(e^{-u^2}\big) + O\big(e^{-v^2}\big),\label{eq:prop1_eq3}
\end{align}
where we use~\eqref{eq:prop1_eq1},~\eqref{eq:prop1_eq2} and Lemma~\ref{lem:bd_rademacher_real} to derive~\eqref{eq:prop1_eq3}. Taking a union bound over $\ell \in [m]$, we obtain with probability at least $1-O\big(m \big(e^{-u^2}+e^{-v^2}\big)\big)$ that 
\begin{equation}
 \max_{\ell \in [m]}  \left|z_\ell \sum_{i \ne \ell} w_{i,\ell} \theta_i(z_i)\right| \le \frac{Cv}{\sqrt{m(t+\log m)}} u.
\end{equation}
Then, we have
\begin{align}
 &\left|\frac{1}{m}\sum_{i \in [m]} \theta_i(z_i) \sum_{\ell \ne i} w_{i,\ell} z_\ell\right|  = \left|\frac{1}{m}\sum_{\ell \in [m]} z_\ell  \sum_{i \ne \ell} w_{i,\ell} \theta_i(z_i)\right| \\
 & \le \max_{\ell \in [m]}  \left|z_\ell \sum_{i \ne \ell} w_{i,\ell} \theta_i(z_i)\right| \le \frac{Cv}{\sqrt{m(t+\log m)}} u.
\end{align}
Setting $u = C' \sqrt{t+\log m}$ and $v = C'' \sqrt{t+\log m}$, we obtain the desired result.
\end{proof}

\begin{proposition}\label{prop:prop2}
 Let $\bz \in \bbR^m$ be a standard Gaussian vector. Let $\bt \sim \calN(\bm{0},\bm{\Sigma})$ be independent of $\bz$, where $\bm{\Sigma} \in \bbR^m$ satisfies $\Sigma_{i,i} \le 1$ for $i \in [m]$ and $\Sigma_{i_1,i_2} = O\big(\frac{1}{m}\big)$ for $i_1 \ne i_2$. Then, for any $\epsilon \in (0,1)$, with probability $1-e^{-\Omega(m \epsilon^2)}$,
 \begin{equation}
  \left|\frac{1}{m}\sum_{i \in [m]} t_i \theta_i(z_i) \right| \le \epsilon.
 \end{equation}
 \end{proposition}
\begin{proof}
 We have 
 \begin{equation}\label{eq:prop2_eq1}
  \|\bSigma\|_\rmF^2 = \sum_{i \in [m]} \Sigma_{i,i}^2 + \sum_{i_1 \ne i_2 } \Sigma_{i_1,i_2}^2 \le m + m^2 \cdot O\big(\frac{1}{m^2}\big) = O(m),
 \end{equation}
and from Lemma~\ref{lem:riesz_thorin}, 
\begin{equation}\label{eq:prop2_eq2}
 \|\bSigma\|_{2\to 2} \le \|\bSigma\|_{1\to 1} = O(1).
\end{equation}
Then, setting $\bB= \bSigma$ in Lemma~\ref{lem:hanson_wright}, we obtain with probability at least $1-e^{-\Omega(m)}$ that
\begin{equation}\label{eq:prop2_eq3}
 \sum_{i \in [m]} t_i^2 \le  2m .
\end{equation}
Again, note that $\theta_i(z_0),\theta_i(z_1),\ldots,\theta_i(z_{m-1})$ form a Rademacher sequence. For any $u >0$ and a sufficiently large absolute constant $C$, we have
\begin{align}
 \bbP\left(\left|\frac{1}{m}\sum_{i\in [m]} t_i \theta_i(z_i)\right| > \frac{C }{\sqrt{m}} u\right) &\le \bbP\left(\sum_{i \in [m]} t_i^2 > 2m\right)+ \bbP\left(\left|\frac{1}{m}\sum_{i\in [m]} t_i \theta_i(z_i)\right| > \frac{C }{\sqrt{m}} u \text{ }\Big| \text{ }\sum_{i \in [m]}t_i^2 \le 2m\right)  \\
 & \le O\big(e^{-u^2}\big) + e^{-\Omega(m)},\label{eq:prop2_eq4}
\end{align}
where we use Lemma~\ref{lem:bd_rademacher_real} and~\eqref{eq:prop2_eq3} to derive~\eqref{eq:prop2_eq4}. Setting $\epsilon = \frac{C}{\sqrt{m}} u$, and we obtain the desired result.
\end{proof}

\subsection{Proof of Theorem~\ref{thm:corr_circulant}}
\label{sec:proof_main_theorem}

Before proving Theorem~\ref{thm:corr_circulant}, we provide some further lemmas. %In the following, we focus on the cases of sign flips and Gaussian noise. The case of logit noise is analogous, and is omitted to avoid repetition. 
\begin{lemma}\label{lem:inner_upper_bd}
Suppose that the assumptions in Lemma~\ref{lem:essential_circ} are satisfied. Then, for any $\bv \in \bbR^n$ and $t > 0$ satisfying $m = \Omega(t + \log n)$, with probability $1-e^{-\Omega(t)} - e^{-\Omega(\sqrt{m})}$, we have
 \begin{equation}
  \left|\left\langle \frac{1}{m} \tilde{\bA}^T\bb - \lambda \tilde{\bx}^*, \bD_{\bxi}\bv\right \rangle\right| \le O\left(\|\bv\|_2\sqrt{\frac{t + \log n}{m}}\right).
 \end{equation}
\end{lemma}
\begin{proof}
 Using the same notations for $\bs_i,\bu_i,\br_i, h_i$ as those in the proof of Lemma~\ref{lem:essential_circ}, we have
 \begin{align}
  &\left\langle \frac{1}{m} \tilde{\bA}^T\bb - \lambda \tilde{\bx}^*, \bD_{\bxi}\bv\right \rangle = \frac{\|\bv\|_2}{m}\sum_{i \in [m]} \left(\langle \tilde{\ba}_i,\bD_{\bxi}\bar{\bv}\rangle b_i - \lambda \langle \tilde{\bx}^*, \bD_{\bxi}\bar{\bv}\rangle\right) \\
  & = \frac{\|\bv\|_2}{m} \sum_{i \in J_{\eta}} \left(\langle \bg,\bw_i\rangle \theta(\langle \bg,\bu_i\rangle) - \lambda \left\langle \bw_i, \frac{\bu_i}{\|\bu_i\|_2}\right\rangle\right) \nonumber \\
  & \indent - \frac{\lambda \|\bv\|_2}{m} \sum_{i \in J_{\eta}} \left(\langle \tilde{\bx}^*,\bD_{\bxi}\bar{\bv}\rangle - \left\langle \bw_i,\frac{\bu_i}{\|\bu_i\|_2} \right\rangle\right) \nonumber \\
  & \indent + \frac{\|\bv\|_2}{m} \sum_{i \in I_\eta} \left(\langle \bg,\bw_i\rangle b_i -\lambda \langle \tilde{\bx}^*,\bD_{\bxi}\bar{\bv}\rangle \right),\label{eq:v_threeTerms}
 \end{align}
where $\bar{\bv} := \frac{\bv}{\|\bv\|_2}$ and $\bw_i := s_{i \leftarrow} (\bD_{\bxi}\bar{\bv})$ so that $\langle \tilde{\bx}^*,\bD_{\bxi}\bar{\bv}\rangle = \langle s_{i \leftarrow}(\tilde{\bx}^*), s_{i \leftarrow} (\bD_{\bxi}\bar{\bv}) \rangle = \langle \bs_i,\bw_i \rangle$. Note that~\eqref{eq:v_threeTerms} is identical to~\eqref{eq:three_terms}, except that we use $\langle \bg,\bw_i \rangle$ to replace $g_{j-i}$ therein. Similarly to~\eqref{eq:complex_four_terms}, we know that we only need to focus on bounding $\frac{1}{m} \sum_{i \in [m]} \big(\langle \bg,\bw_i\rangle \theta(\langle \bg,\bu_i\rangle) - \lambda \big\langle \bw_i, \frac{\bu_i}{\|\bu_i\|_2}\big\rangle\big)$. Setting $\alpha_{i,\ell} := \mathrm{Cov}[\langle \bg,\bw_i\rangle, h_\ell] = \big\langle \bw_i,\frac{\bu_\ell}{\|\bu_\ell\|_2}\big\rangle$, similarly to~\eqref{eq:g_jminusi}, $\langle \bg,\bw_i\rangle$ can be written as 
\begin{equation}
 \langle \bg,\bw_i\rangle = \sum_{\ell \in [m]} \alpha_{i,\ell} h_\ell + \gamma_\ell f_\ell,
\end{equation}
where $\gamma_\ell = \sqrt{1-\sum_{\ell \in [m]}\alpha_{i,\ell}^2}$ and $f_\ell \sim \calN(0,1)$ is independent of $h_0,h_1,\ldots,h_m$. Note that by Lemma~\ref{lem:almost_ortho}, we have with probability at least $1-e^{-t}$ that 
\begin{equation}
 \max_{i \ne \ell} \langle \bw_i, \bs_\ell \rangle \le \rho \cdot O(\sqrt{t + \log m}) = O\left(\frac{1}{m\sqrt{t + \log m}}\right),\label{eq:inner_explain1}
\end{equation}
where we use $\rho = O\big(\frac{1}{m(t+\log m)}\big)$ in~\eqref{eq:inner_explain1}. 
By the $(m,\beta)$-orthogonality ({\em cf.} the first paragraph in the proof of Lemma~\ref{lem:essential_circ} for the cases of random flips and Gaussian noise), for all $i,\ell \in [m]$, we have
\begin{equation}
 \sum_{j \in [n]} \bar{v}_j^2 \br_\ell(j+i)^2 \le \|\br_\ell\|_\infty^2 = O\left(\frac{1}{m^2 (t+\log m)}\right).
\end{equation}
Setting $u = O(\sqrt{t+\log m})$ in Lemma~\ref{lem:bd_rademacher_real} and taking a union bound over $i,\ell \in [m]$, we obtain with probability at least $1-e^{-t}$ that
\begin{align}
 &\max_{i, \ell \in [m]}\left|\langle \bw_i, \br_\ell \rangle\right| = \max_{i, \ell \in [m]}\left|\left\langle \bD_{\bxi}\bar{\bv}, s_{\rightarrow i}(\br_\ell) \right\rangle \right| \\
 & = \max_{i, \ell \in [m]}|\sum_{j \in [n]} \bar{v}_j \br_\ell(j+i)\xi_j| \le O\big(\frac{1}{m}\big).\label{eq:inner_explain2}
\end{align}
Therefore, we obtain from~\eqref{eq:inner_explain1} and~\eqref{eq:inner_explain2} that
\begin{align}
 &\max_{i\ne \ell} |\alpha_{i,\ell}| \le 2\max_{i\ne \ell} |\langle \bw_i, \bu_\ell\rangle | = 2\max_{i \ne \ell} |\langle \bw_i, \bs_\ell -\br_\ell\rangle | = O\big(\frac{1}{m}\big),
\end{align}
and we are able to derive inequalities similar to~\eqref{eq:prop2_eq1},~\eqref{eq:prop2_eq2} and~\eqref{eq:prop2_eq3}.
Following the same proof techniques as those for Lemma~\ref{lem:essential_circ}, we obtain the desired result. 
\end{proof}

Based on Lemmas~\ref{lem:essential_circ} and~\ref{lem:inner_upper_bd}, and using a well-established chaining argument that we do not repeat here (e.g., see~\cite[Lemma~3]{liu2020generalized}), we attain the following lemma. Note that we use the $(m,\beta)$-orthogonality  and the event $\calE$ ({\em cf.~\eqref{eq:event_e}}) only once, and do not need to take any union bound for them.
\begin{lemma}
\label{lem:rhs_bd}
 Under the conditions in Theorem~\ref{thm:corr_circulant}, with probability at least $1-e^{-\Omega( k\log \frac{Lr}{\delta})}- e^{-\Omega(\sqrt{m})}$, it holds that
 \begin{align}
  &\left|\left\langle \frac{1}{m} \tilde{\bA}^T\bb - \lambda \tilde{\bx}^*, \bD_{\bxi}\hat{\bx} - \tilde{\bx}^*\right \rangle\right| \le O\left(\sqrt{\frac{k \log\frac{Lr}{\delta}}{m}}\right) (\|\hat{\bx}-\bx^*\|_2 + \delta).
 \end{align}
\end{lemma}

The restricted isometry property (RIP) is widely used in compressive sensing (CS), and we present its definition in the following.
\begin{definition}
 A matrix $\bPhi \in \bbC^{m \times n}$ is said to have the restricted isometry property of order $s$ and level $\delta \in (0,1)$ (equivalently, $(s,\delta)$-RIP) if
 \begin{equation}\label{eq:s_RIP}
  (1-\delta) \|\bx\|_2^2 \le \|\bPhi \bx\|_2^2 \le (1+\delta) \|\bx\|_2^2 
 \end{equation}
for all $s$-sparse vectors in $\bbC^n$. The restricted isometry constant $\delta_s$ is defined as the smallest value of $\delta$ for which~\eqref{eq:s_RIP} holds.  
\end{definition}

We have the following lemma that guarantees the RIP for partial Gaussian circulant matrices.
\begin{lemma}{\em (\hspace{1sp}Adapted from~\cite[Theorem~4.1]{krahmer2014suprema})}\label{lem:RIP_circulant}
 For $s \le n$ and $\eta,\delta \in (0,1)$, if 
 \begin{equation}\label{eq:sample_complexity_RIP}
  m \ge \Omega\left(\frac{s}{\delta^2} \left(\log^2 s\right) \left(\log^2 n\right) \right),
 \end{equation}
then with  probability at least $1-e^{-\Omega((\log^2 s) (\log^2 n))}$, the restricted isometry constant of $\frac{1}{\sqrt{m}}\tilde{\bA} = \frac{1}{\sqrt{m}}\bR_{\Omega} \bC_{\bg}$ satisfies $\delta_s \le \delta$. 
\end{lemma}

Based on the RIP, we have the Johnson-Lindenstrauss embedding according to the following lemma. 
\begin{lemma}{\em (\hspace{1sp}\cite[Theorem~3.1]{krahmer2011new})}
\label{lem:RIP_JL}
 Fix $\eta >0$ and $\epsilon \in (0,1)$, and consider a finite set $E \subseteq \bbC^n$ of cardinality $|E| = p$. Set $s \ge 40 \log \frac{4p}{\eta}$, and suppose that $\bPhi \in \bbC^{m \times  n}$ satisfies the RIP of order $s$ and level $\delta \le \frac{\epsilon}{4}$. Then, with probablity at least $1-\eta$, we have
 \begin{equation}
  (1-\epsilon) \|\bx\|_2^2 \le \|\bPhi \bD_{\bxi} \bx\|_2^2 \le (1+\epsilon) \|\bx\|_2^2
 \end{equation}
uniformly for all $\bx \in E$, where $\bD_{\bxi} = \mathrm{Diag}(\bxi)$ is a diagonal matrix with respect to a Rademacher sequence $\bxi$.  
\end{lemma}

Combining Lemmas~\ref{lem:RIP_circulant} and~\ref{lem:RIP_JL} with setting $s = \Omega(\log p)$ and $\eta = O\big(\frac{1}{p}\big)$, we derive the following corollary. 
\begin{corollary}\label{coro:RIP_JL}
 Fix $\epsilon \in (0,1)$, and consider a finite set $E \subseteq \bbC^n$ of cardinality $|E| = p$ satisfying $n =\Omega(\log p)$. Suppose that 
 \begin{equation}
  m = \Omega\left(\frac{\log p}{\epsilon^2} \left(\log^2 (\log p)\right) \left(\log^2 n\right)\right).
 \end{equation}
 Then, with probablity $1- e^{- \Omega(\log p)}- e^{-\Omega((\log^2 (\log p)) (\log^2 n))}$, we have
 \begin{equation}
  (1-\epsilon) \|\bx\|_2^2 \le \left\|\frac{1}{\sqrt{m}} \bA \bx\right\|_2^2 \le (1+\epsilon) \|\bx\|_2^2
 \end{equation}
 for all $\bx \in E$, where we recall that $\bA = \bR_{\Omega} \bC_{\bg}\bD_{\bxi} = \tilde{\bA}\bD_{\bxi}$ ({\em cf.}~\eqref{eq:partial_circ_sign}). 
\end{corollary}

Note that from Lemma~\ref{lem:large_dev_Gaussian} and a union bound over $[n]$, and by Lemma~\ref{lem:riesz_thorin}, we have that for any $t>0$, with probability $e^{-\Omega(t)}$, $\big\|\frac{1}{\sqrt{m}}\bA\big\|_{2\to 2} \le \frac{n\sqrt{t+\log n}}{\sqrt{m}}$. Based on this bound for the spectral norm of $\frac{1}{\sqrt{m}}\bA$ and Corollary~\ref{coro:RIP_JL}, and by again using a well-established chaining argument (e.g., see~\cite[Lemma~4.1]{bora2017compressed},~\cite[Section~3.4]{daras2021intermediate}), we obtain the following lemma that guarantees the two-sided Set-Restricted
Eigenvalue Condition (S-REC)~\cite{bora2017compressed} for partial Gaussian circulant matrices. This lemma will only be used to derive an upper bound for the term corresponding to adversarial noise. 
\begin{lemma}\label{lem:S-REC_circulant}
 For any $\delta >0$ satisfying $Lr =\Omega(\delta n)$
 and $n = \Omega\big(k \log\frac{Lr}{\delta}\big)$, and any $\alpha \in (0,1)$, if
 \begin{equation}
  m = \Omega\left(\frac{k \log\frac{Lr}{\delta} }{\alpha^2} \left(\log^2 \left(k \log\frac{Lr}{\delta}\right)\right) \left(\log^2 n\right)\right),
 \end{equation}
then with probability $1- e^{- \Omega(k \log \frac{Lr}{\delta})}- e^{-\Omega((\log^2 (k\log \frac{Lr}{\delta})) (\log^2 n))}$, it holds that
\begin{align}
 (1-\alpha)\|\bx_1 -\bx_2\|_2 -\delta &\le \left\|\frac{1}{\sqrt{m}}\bA (\bx_1-\bx_2)\right\|_2 \le (1+\alpha)\|\bx_1 -\bx_2\|_2 +\delta,
\end{align}
for all $\bx_1,\bx_2 \in G(B_2^k(r))$.
 \end{lemma}

We now provide the proof of Theorem~\ref{thm:corr_circulant}.
\begin{proof}[Proof of Theorem~\ref{thm:corr_circulant}]
 Because $\hat{\bx}$ is a solution to~\eqref{eq:opt_corr} and $\bx^* \in \calK:= G(B_2^k(r))$, we have 
\begin{equation}
 \frac{1}{m} \tilde{\bb}^T (\bA \hat{\bx}) \ge \frac{1}{m} \tilde{\bb}^T (\bA \bx^*).
\end{equation}
Recall that $\tilde{\bA} = \bR_\Omega \bC_\bg$, and $\tilde{\bx}^* = \bD_{\bxi} \bx^*$ ({\em cf.} Section~\ref{sec:main_res}). We have $\bA =\bR_\Omega \bC_\bg \bD_{\bxi} =  \tilde{\bA}\bD_{\bxi}$ and
\begin{equation}
 \frac{1}{m} \tilde{\bb}^T (\tilde{\bA} \bD_{\bxi}\hat{\bx}) \ge \frac{1}{m} \tilde{\bb}^T (\tilde{\bA} \tilde{\bx}^*),
\end{equation}
which can equivalently be expressed as
\begin{equation}
 \left\langle \frac{1}{m} \tilde{\bA}^T\tilde{\bb} - \lambda \tilde{\bx}^*, \bD_{\bxi}\hat{\bx} - \tilde{\bx}^*\right \rangle + \langle \lambda \tilde{\bx}^*, \bD_{\bxi}\hat{\bx} - \tilde{\bx}^*\rangle \ge 0.
\end{equation}
By the assumption that $G(B_2^k(r)) \subseteq \calS^{n-1}$, we obtain that $\|\bD_{\bxi}\hat{\bx}\|_2 = \|\hat{\bx}\|_2 = 1$, and
\begin{equation}
 \langle \lambda \tilde{\bx}^*,  \tilde{\bx}^* - \bD_{\bxi}\hat{\bx}\rangle = \lambda \langle \bx^*, \bx^* - \hat{\bx} \rangle = \frac{\lambda}{2}\|\bx^* - \hat{\bx}\|_2^2.
\end{equation}
Therefore, we obtain
\begin{align}
 \frac{\lambda}{2}\|\bx^* - \hat{\bx}\|_2^2  &\le \left|\left\langle \frac{1}{m} \tilde{\bA}^T(\tilde{\bb}-\bb), \bD_{\bxi}\hat{\bx} - \tilde{\bx}^*\right \rangle\right| + \left|\left\langle \frac{1}{m} \tilde{\bA}^T\bb - \lambda \tilde{\bx}^*, \bD_{\bxi}\hat{\bx} - \tilde{\bx}^*\right \rangle\right| \\
 & \le \left\|\frac{1}{\sqrt{m}}(\tilde{\bb}-\bb)\right\|_2 \cdot\left\|\frac{1}{\sqrt{m}} \bA (\hat{\bx}-\bx^*)\right\|_2 + \left|\left\langle \frac{1}{m} \tilde{\bA}^T\bb - \lambda \tilde{\bx}^*, \bD_{\bxi}\hat{\bx} - \tilde{\bx}^*\right \rangle\right|.\label{eq:thm1_eq1}
\end{align}
Setting $\alpha = \frac{1}{2}$ in Lemma~\ref{lem:S-REC_circulant}, we have that if $m = \Omega(k \log\frac{Lr}{\delta}  (\log^2 (k \log\frac{Lr}{\delta})) (\log^2 n))$, with probability at least $1-e^{-\Omega(k\log \frac{Lr}{\delta})} -e^{-\Omega((\log^2 (k\log \frac{Lr}{\delta})) (\log^2 n))}$,
\begin{equation}
 \left\|\frac{1}{\sqrt{m}} \bA (\hat{\bx}-\bx^*)\right\|_2 \le O(\|\hat{\bx}-\bx^*\|_2 + \delta).\label{eq:thm1_eq2}
\end{equation}
Moreover, by Lemma~\ref{lem:rhs_bd}, with probability $1-e^{-\Omega(k\log \frac{Lr}{\delta})} -e^{-\Omega(\sqrt{m})}$, we have
\begin{align}
 & \left|\left\langle \frac{1}{m} \tilde{\bA}^T\bb - \lambda \tilde{\bx}^*, \bD_{\bxi}\hat{\bx} - \tilde{\bx}^*\right \rangle\right| \le O\left(\sqrt{\frac{k \log\frac{Lr}{\delta}}{m}}\right) (\|\hat{\bx}-\bx^*\|_2 +\delta).\label{eq:thm1_eq3}
\end{align}
Recall that we assume $\frac{1}{\sqrt{m}}\|\tilde{\bb}-\bb\|_2 \le \varsigma$. Combining~\eqref{eq:thm1_eq1},~\eqref{eq:thm1_eq2} and~\eqref{eq:thm1_eq3}, we obtain with probability at least $1-e^{-\Omega(k\log \frac{Lr}{\delta})} -e^{-\Omega((\log^2 (k\log \frac{Lr}{\delta})) (\log^2 n))} - e^{-\Omega(\sqrt{m})}$ that
\begin{equation}\label{eq:two_terms_larger}
 \frac{\lambda}{2}\|\bx^* - \hat{\bx}\|_2^2 \le O\left(\sqrt{\frac{k \log\frac{Lr}{\delta}}{m}} + \varsigma\right) (\|\hat{\bx}-\bx^*\|_2 +\delta).
\end{equation}
Then, by considering both possible cases of which of the two terms in \eqref{eq:two_terms_larger} is larger, we find that the desired result holds when $\delta = O\big(\frac{\sqrt{k\log\frac{Lr}{\delta}}}{\lambda\sqrt{m}}+\frac{\varsigma}{\lambda}\big)$.
\end{proof}

\section{Numerical Experiments}
\label{sec:exp}

We empirically evaluate the correlation-based decoder on the MNIST~\cite{lecun1998gradient} and celebA~\cite{liu2015deep} datasets. The MNIST dataset consists of 60,000 handwritten images of size $ n = 28 \times 28 = 784$. The CelebA dataset is a large-scale face attributes dataset of more than 200,000 images of celebrities. The input images were cropped to a $64 \times 64$ RGB image, yielding a dimensionality of $n = 64 \times 64 \times 3 = 12288$.  
For simplicity, we assume that there is no adversarial noise. We make use of the practical iterative algorithm proposed in~\cite{liu2020sample}:
\begin{equation}
    \bx^{(t+1)}  = \calP_G \left(\bx^{(t)} + \mu \bA^T(\bb - \mathrm{sign}(\bA\bx^{(t)}))\right), \label{eq:pgd}
\end{equation}
where $\calP_G(\cdot)$ is the projection function onto $G(B_2^k(r))$, $\bb$ is the uncorrupted observation vector, $\bx^{(0)} = \mathbf{0}$, and $\mu >0$ is a parameter. As the performance difference among all three noisy 1-bit observation models (random bit flips, and Gaussian or logit noise added before quantization; {\em cf.} Section~\ref{sec:setup}) is mild, we only present the numerical results corresponding to the case of Gaussian noise. 

We compare with the 1-bit Lasso~\cite{plan2013one}, and BIHT~\cite{jacques2013robust}, which are sparsity-based algorithms. The sensing matrix $\mathbf{A}$ is either a randomly signed partial Gaussian circulant matrix or a standard Gaussian matrix. On the MNIST dataset, a generative model based on variational auto-encoder (VAE) is used and thus our method is denoted by \texttt{1b-VAE-C} or \texttt{1b-VAE-G}. where ``C'' refers to the partial Gaussian circulant matrix and ``G'' refers to the standard Gaussian matrix. 
The latent dimension of the VAE is $k = 20$. On the celebA dataset, a generative model based on Deep Convolutional Generative Adversarial Networks (DCGAN) is used and thus our method is denoted by \texttt{1b-DCGAN-C} or \texttt{1b-DCGAN-G}, respectively. The latent dimension of the DCGAN is $k = 100$.

We follow \cite{liu2020sample} to perform projected gradient descent (PGD) to maximize a constrained correlation-based objective function. The number of total iterations is $300$ with $\mu=0.2$. We use Adam optimizer to perform the projection with adaptive number of steps and learning rates based on the changes of $\bx^{(t)}$. When $\bx^{(t)}$ changes dramatically in the beginning phase of the optimization, we use a large number of Adam optimization steps with a large learning rate. When $\bx^{(t)}$ has minor change in the later phase of optimization, we use a small number of Adam optimization steps with a small learning rate. Specifically, in the first $20$ iterations, the number of Adam optimization steps is $30$ with learning rate $0.3$. From the $20$-th iteration to $100$-th iteration, the number of Adam optimization steps is $10$ with learning rate $0.2$. From the $100$-th iteration to $300$-th iteration, the number of Adam optimization steps is $5$ with learning rate $0.1$. The generative models we use in our experiments have the same architectures as those of \cite{bora2017compressed}. The experiments are done on test images that were not used when training the generative models. To reduce the impact of local minima, we choose the best estimation among $5$ random restarts. On the MNIST dataset, the standard deviation of the Gaussian noise $\sigma$ is chosen as $0.1$. On the celebA dataset, the standard deviation $\sigma$ is chosen as $0.01$. 

We observe from Figures~\ref{fig:recon_mnist},~\ref{fig:recon_celebA}, and~\ref{fig:recon_error} that on both datasets, the sparsity-based methods 1-bit Lasso and BIHT attain poor reconstructions, while the generative model based method~\eqref{eq:pgd} attains mostly accurate reconstructions when the number of measurements is small. From Figures~\ref{fig:recon_mnist} and~\ref{fig:recon_error}-a, we observe that for the MNIST dataset, the reconstruction performance of using a partial Gaussian circulant matrix and a standard Gaussian matrix are almost the same. For the celebA dataset, according to Figures~\ref{fig:recon_celebA} and~\ref{fig:recon_error}-b, the reconstruction performances are again very similar.

 \begin{figure}[!tbp]
 \centering
% \begin{tabular}{cc}
\includegraphics[width = 1.0\columnwidth]{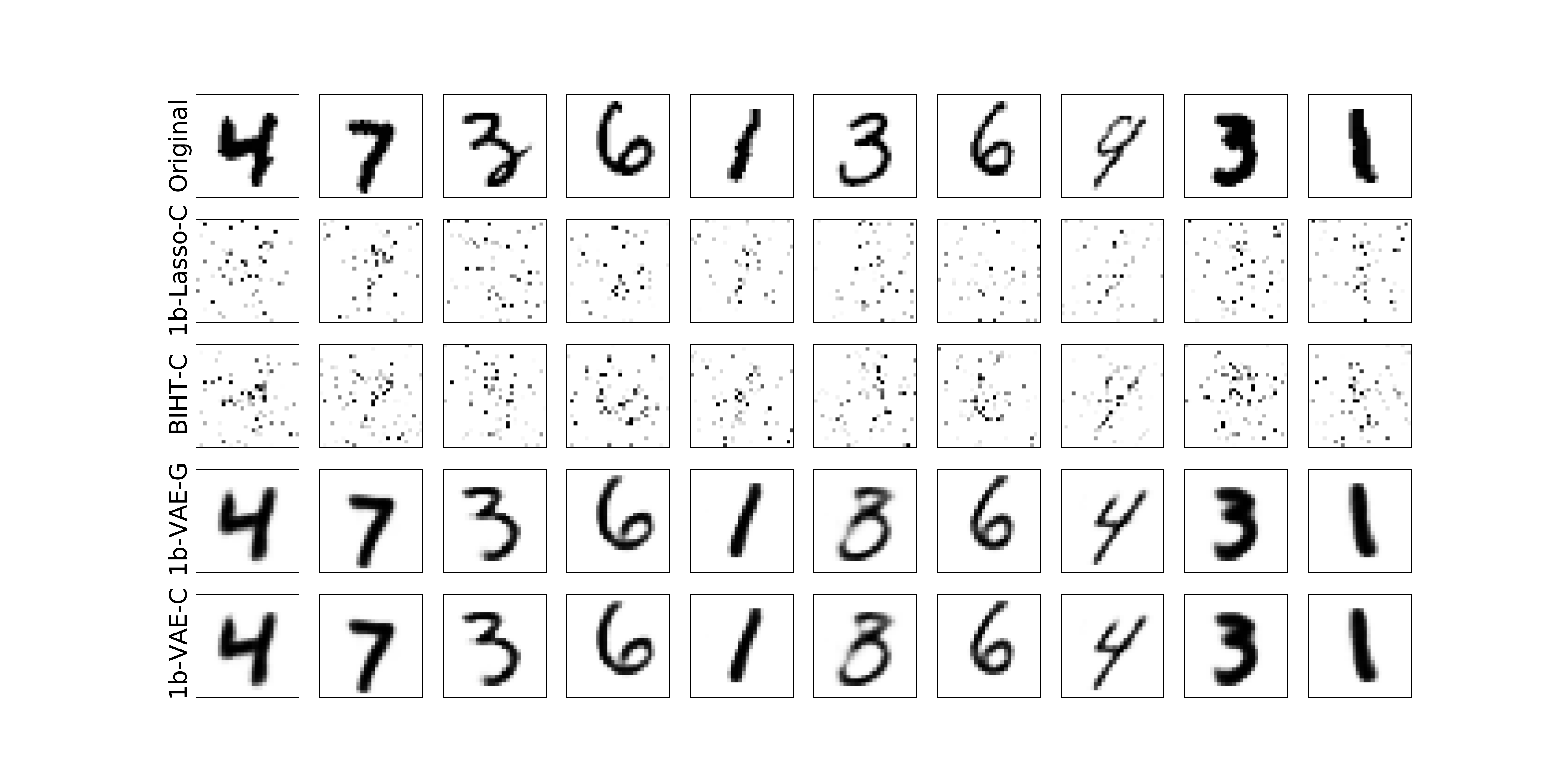}
\caption{Reconstruction results on MNIST with $200$ measurements.\label{fig:recon_mnist}}
\vspace*{-2ex}
% \end{tabular}  
\end{figure}

\begin{figure}[!tbp]
\centering
% \begin{tabular}{cc}
\includegraphics[width = 1.0\columnwidth]{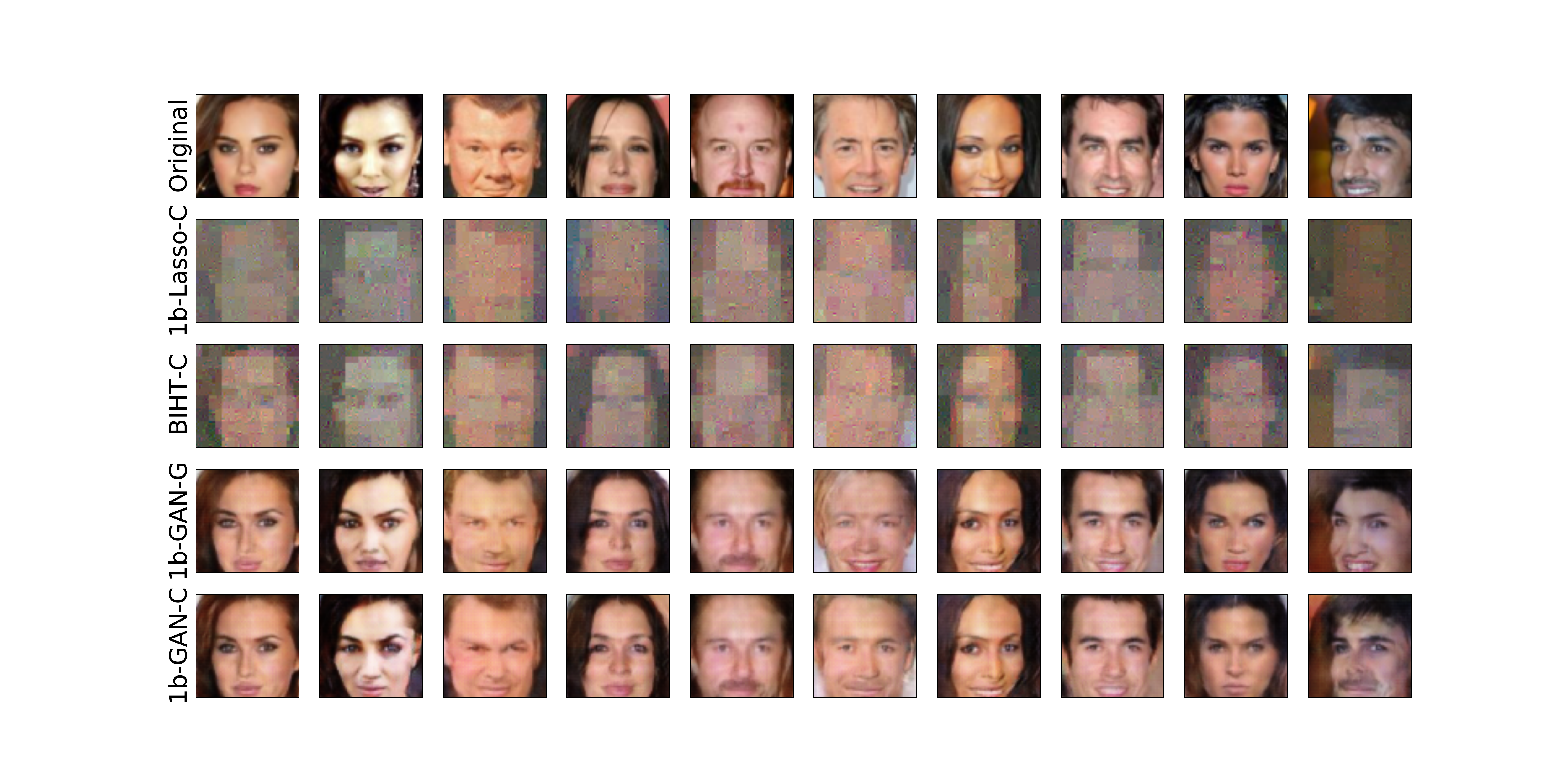}
\caption{Reconstruction results on celebA with $1000$ measurements.\label{fig:recon_celebA}}
% \end{tabular} 
\end{figure}

\begin{figure}[!tbp]
\centering
\begin{tabular}{cc}
\includegraphics[width=.5\columnwidth]{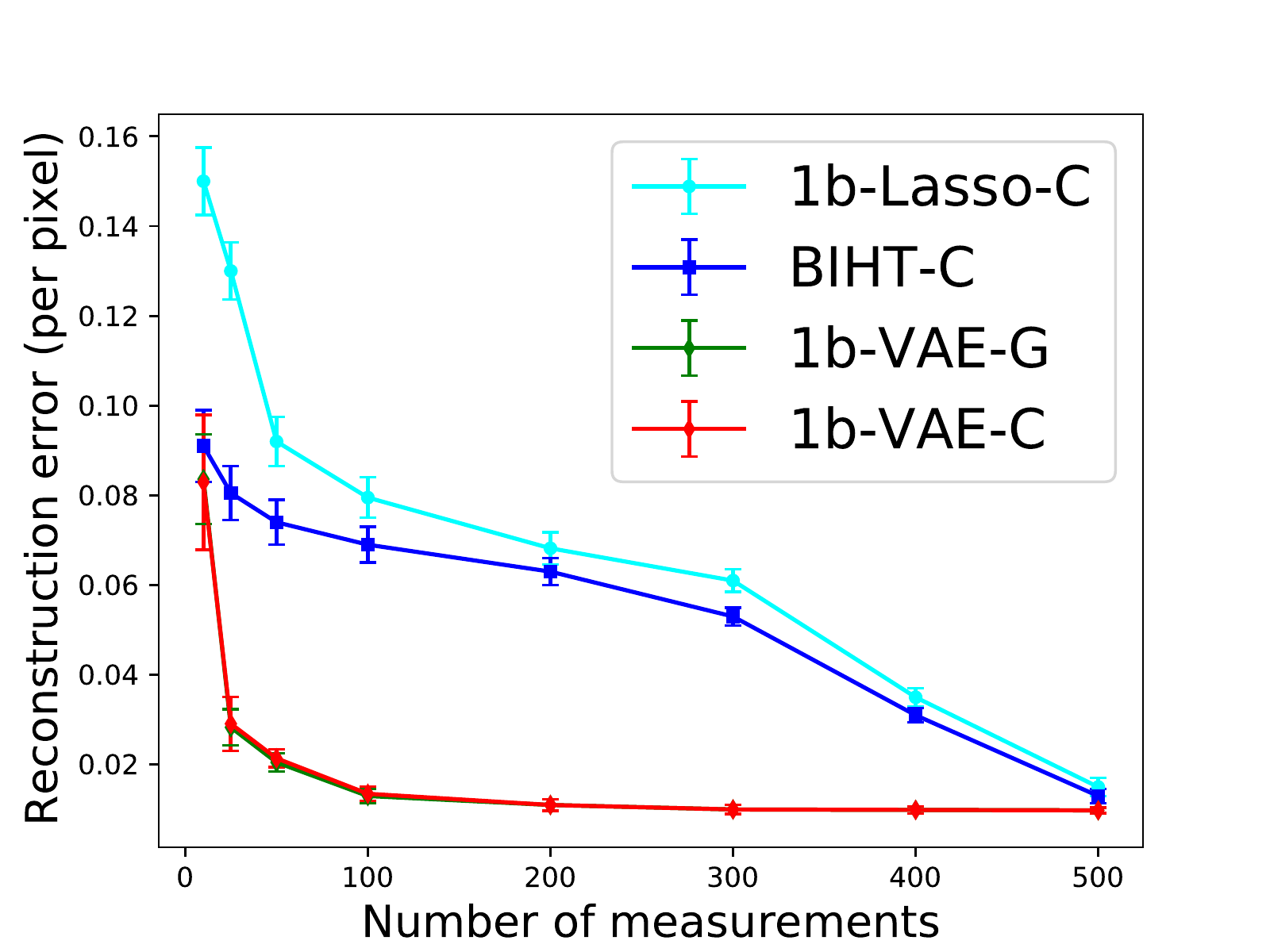} & \hspace{-1.cm}
\includegraphics[width=.5\columnwidth]{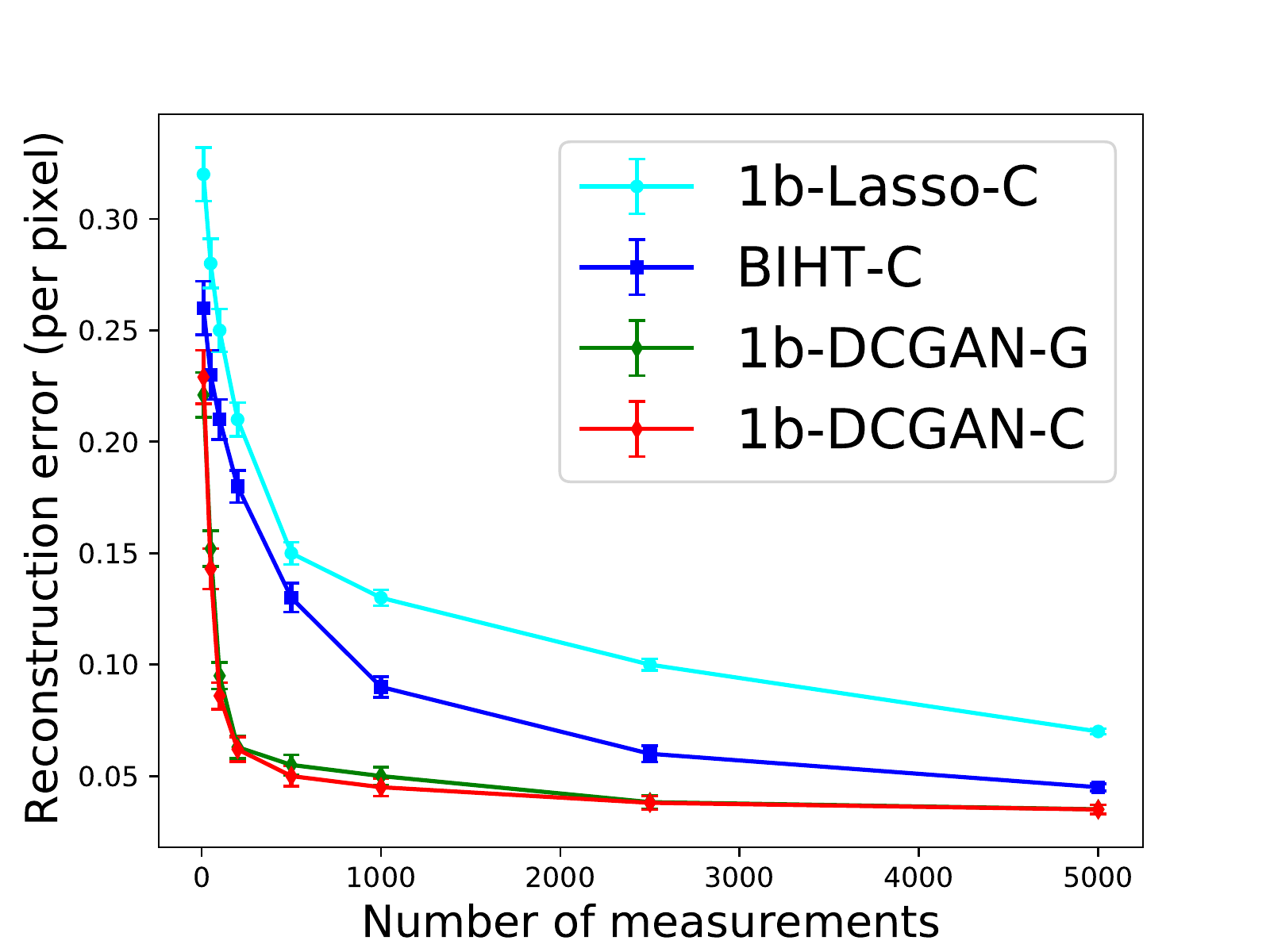} \\
{\small (a) Results on MNIST } & {\small (b) Results on celebA }
\end{tabular}
\caption{The vertical bars indicate 95\% confidence intervals. The reconstruction error is calculated over $1000$ images by averaging the per-pixel error in terms of $\ell_2$ norm.\label{fig:recon_error}}
\end{figure}

We also compare the running time of each iteration of~\eqref{eq:pgd} for different sensing matrices (standard Gaussian or partial Gaussian circulant). When $\bA$ is a standard Gaussian matrix, the time complexity of calculating $\bA \bx^{(t)}$ and $\bA^T \bb$ is $O(mn)$. When $\bA$ is a partial Gaussian circulant matrix, we can use FFT to accelerate the  matrix-vector multiplications, and the time complexity is reduced to $O(n\log n)$.  Since the dimensionality of the image vector is relatively small,  %and the number of measurements are relatively small, 
and the projection operator $\calP_G(\cdot)$ is also time consuming, in our experiments, the speedup of using a partial Gaussian circulant matrix is not as significant as that in~\cite[Table 2]{yu2017binary}. However, from Table~\ref{table:running_time}, we can still observe that using a partial Gaussian circulant matrix is typically about $2$ times faster than using a standard Gaussian matrix, and for the celebA dataset with ambient dimension $n = 12288$, when $m = 5000$, using partial Gaussian circulant matrices leads to around $4$ times speedup compared with using standard Gaussian matrices. We expect a greater speedup for settings with higher dimensionality.
%\begin{table}[t]
%\begin{center}
%\begin{tabular}{|c|c|c|} 
% \hline
%  & MNIST & celebA   \\ \hline 
% Gaussian (CPU)  & 53.6 & 758.4  \\ \hline
% Circulant (CPU) & {\bf 37.7} & {\bf 446.9}  \\  \hline
% Gaussian (GPU)  & 3.9 & 41.2  \\ \hline
% Circulant (GPU) & {\bf 2.8} & {\bf 18.6}  \\  \hline
%\end{tabular}
%\caption{\label{time} Average time comparison (ms) of different sensing matrices per iteration in~\eqref{eq:pgd}. The number of measurements is fixed to be $200$. We compare the running time in two types of resource: A single 2.9GHz CPU core, and a Nvidia GPU (Tesla K80).} \label{table:running_time}%The majority of computation cost comes from $x=G(z)$ and $Ax.$}
%\end{center}
%\end{table}
\begin{table}[t]
\begin{center}
\begin{tabular}{|c|c|c|c|c|} 
 \hline
  & \multicolumn{2}{|c|}{MNIST ($n=784$)} & \multicolumn{2}{|c|}{celebA ($n=12288$)}   \\ 
   \cline{2-5} & m=200 & m=700 & m=200 & m=5000 \\ \hline 
 Gaussian   & 51.4 & 113.6 & 704.8 & 5632.6 \\ \hline
Circulant  & {\bf 33.6} & {\bf 57.8} & {\bf 378.3} & {\bf 1367.5} \\  \hline
\end{tabular}
\caption{\label{time} Average time comparison (ms) of different sensing matrices per iteration in~\eqref{eq:pgd}. We compare the running time in a single 2.9GHz CPU core. GPU can also be applied to further speed up the computation. ~\label{table:running_time}}
\end{center}
\end{table}

\section{Conclusion and Future Work}

We have established a sample complexity upper bound for noisy 1-bit compressed sensing for the case of using a randomly signed partial Gaussian circulant sensing matrix and a generative prior. The sample complexity upper bound matches that obtained previously for i.i.d.~Gaussian sensing matrices, and our scheme is robust to adversarial noise. 

We focused only on non-uniform recovery guarantees, and uniform recovery is an immediate direction for further research.  Relaxing the $\ell_{\infty}$ assumption and allowing for general scalings of $k$ (not only $k \ll n^{1/4}$) would also be of significant interest.  Finally, handling Fourier measurements without column sign flips would be of significant interest for applications where the measurement matrix is constrained as such.

\appendices

\section{Proof of Lemma~\ref{lem:gamma_ortho}}\label{sec:proof_gamma_ortho}

For any $\varepsilon >0$, from Lemma~\ref{lem:almost_ortho} and the union bound, with probability at least $1-m^2 e^{-\frac{\varepsilon^2}{8\rho^2}}$, it holds that
 \begin{equation}
  |\langle \bs_i,\bs_j \rangle| \le  \varepsilon
 \end{equation}
for all $0\le i \ne j<m$. Setting $m^2 e^{-\frac{\varepsilon^2}{8\rho^2}} = \nu$, we obtain 
\begin{equation}\label{eq:r_infty_bd}
 \varepsilon = 2\sqrt{2}\rho \sqrt{\log \frac{m^2}{\nu}}.
\end{equation}
 Let $\bB = [\bs_0,\bs_1,\ldots,\bs_{m-1}] \in \bbR^{n \times m}$. Then, all the diagonal entries of $\bB^T\bB$ are $1$, and the magnitude of all the off-diagonal entries are upper bounded by $\varepsilon$. By the Gershgorin circle theorem~\cite{van1983matrix}, we have 
\begin{equation}
 \sigma_{m}(\bB^T\bB) \ge 1-(m-1)\varepsilon > 1-m\varepsilon,\label{eq:similarly_used}
\end{equation}
where $\sigma_{m}(\cdot)$ denotes the $m$-th largest singular value of a matrix. For any $i \in \{1,\ldots,m-2\}$, and any unit vector $\bt$ in the span of $\bs_0,\ldots,\bs_{i-1}$, we write $\bt$ as $\bt = \sum_{j<i} \alpha_j \bs_j$. Let $\balpha_i = [\alpha_0,\ldots,\alpha_{i-1}]^T \in \bbR^i$ and $\bB_i = [\bs_0,\bs_1,\ldots,\bs_{i-1}] \in \bbR^{n \times i}$. Similarly to~\eqref{eq:similarly_used}, we have 
\begin{equation}
 \sigma_{i}(\bB_i^T\bB_i) \ge 1-(i-1)\varepsilon > 1-m\varepsilon.
\end{equation}
Therefore, we obtain
\begin{align}
 1 &= \|\bB_i \balpha_i\|_2^2 \ge \sigma_{i}^2(\bB_i) \|\balpha_i\|_2^2 = \sigma_{i}(\bB_i^T\bB_i)\|\balpha_i\|_2^2 \ge (1-m\varepsilon)\|\balpha_i\|_2^2,
\end{align}
and hence
\begin{equation}
 \|\balpha_i\|_2^2 \le \frac{1}{1-m\varepsilon}.
\end{equation}
In addition, 
\begin{align}
 &\langle \bs_i,\bt \rangle^2 = \left(\sum_{j < i} \alpha_j \langle \bs_i, \bs_j \rangle\right)^2 \le \|\balpha_i\|_2^2 \left(\sum_{j < i}\langle \bs_i, \bs_j \rangle^2\right) \\
 &\le \frac{1}{1-m\varepsilon}\cdot i \varepsilon^2 < \frac{m\varepsilon^2}{1-m\varepsilon}. 
\end{align}
By definition, the squared length of the projection of $\bs_i$ onto the span of $\bs_0,\ldots,\bs_{i-1}$ is equal to $\max\{\langle \bt,\bs_i\rangle^2 \,:\, \bt \in \mathrm{span}\{\bs_0,\ldots,\bs_{i-1}\}, \|\bt\|_2=1\}$. Hence, the projection of $\bs_i$ onto the span of $\bs_0,\ldots,\bs_{i-1}$ has squared length at most $\frac{m\varepsilon^2}{1-m\varepsilon}$. Then, if 
\begin{equation}
 2\sqrt{2}m \rho \log \frac{m^2}{\nu} < \frac{1}{2}, 
\end{equation}
recalling that $\varepsilon = 2\sqrt{2}\rho \sqrt{\log \frac{m^2}{\nu}}$, we also have $m \varepsilon < \frac{1}{2}$, and 
\begin{align}
 &\frac{m\varepsilon^2}{1-m\varepsilon} < 2m\varepsilon^2 = 16 m \rho^2 \log \frac{m^2}{\nu} \\
 &=  (2\sqrt{2} \rho) \times (4\sqrt{2} m \rho \log \frac{m^2}{\nu}) \le 2\sqrt{2} \rho  < 4\rho. \label{eq:r_ell2_bd}
\end{align}
We can now establish Lemma~\ref{lem:gamma_ortho} by performing the following procedure on the vectors (essentially Gram-Schmidt orthogonalization):
\begin{enumerate}
 \item Initialize: $\bu_0 = \bs_0$, $\br_0 = \mathbf{0}$. 
 \item For $i = 2,3,\ldots,m-1$, we set $\br_i$ to be the projection of $\bs_i$ onto the span of $\bs_0,\ldots,\bs_{i-1}$, and $\bu_i = \bs_i - \br_i$.
\end{enumerate}
From~\eqref{eq:r_infty_bd} and~\eqref{eq:r_ell2_bd}, we obtain the desired result.

\section{Table of Notation for the Proof of Lemma~\ref{lem:essential_circ}}
\label{sec:table_proof_lemma1}

For ease of reference, the notation used in proving our most technical result, Lemma~\ref{lem:essential_circ}, is shown in Table \ref{table:notations_lemma1}.

\begin{table*}[t]
\begin{center}
\begin{tabular}{|c|c|} 
 \hline
  Notations & Explanations   \\ \hline 
 $\tilde{\bA} = \bR_{\Omega}\bC_{\bg}$   & A partial Gaussian circulant matrix  \\ \hline 
 $\bA = \bR_{\Omega}\bC_{\bg}\bD_{\bxi}$   & A randomly signed partial Gaussian circulant matrix  \\ \hline 
 $i$   & An integer in $[m]$  \\ \hline 
 $j$   & An integer in $[n]$  \\ \hline 
 $\ell$   & An integer in $[m]$  \\ \hline 
 $\theta$ & $\theta(x) = \tau \cdot\sign(x)$ with $\bbP(\tau = -1) = p$ for random flips and $\theta(x) = \sign(x+e)$ with $e \sim\calN(0,\sigma^2)$ for Gaussian noise\\ \hline 
 $\lambda$ & $\lambda = \bbE[g\theta(g)]$ for $g \sim \calN(0,1)$; $\lambda = (1-2p) \sqrt{\frac{2}{\pi}}$ for random flips and $\lambda = \sqrt{\frac{2}{\pi(1+\sigma^2)}}$ for Gaussian noise\\ \hline 
 $b_i = \theta(\langle\ba_i,\bx^*\rangle)$   & The $i$-th observation (no adversarial noise)  \\ \hline 
 $\tilde{\bx}^* = \bD_{\bxi} \bx^*$   & The product of a diagonal Rademacher matrix and the signal vector $\bx^*$ \\ \hline 
 $\bg$   & A standard Gaussian vector in $\bbR^n$ \\ \hline 
 $\tilde{a}_{ij}$ & The $(i,j)$-th entry of $\tilde{\bA}$; $\tilde{a}_{ij} = g_{j-i}$ \\ \hline
 $\rho$   & An upper bound for $\|\bx^*\|_\infty$, $\rho = O\left(\frac{1}{m(t+\log m)}\right)$ \\ \hline 
 $\beta= 2\sqrt{\rho}$   & Used in the $(m,\beta)$-orthogonality condition; $\beta = O\big(\frac{1}{\sqrt{m(t+\log m)}}\big)$  \\ \hline 
 %$\beta_2 $   & Used in the $(m,\beta_1,\beta_2)$-orthogonality condition; $\beta_2 = O\big(\frac{1}{m\sqrt{(t+\log m)}}\big)$  \\ \hline 
 $\bs_i = s_{i \leftarrow}(\tilde{\bx}^*)$   & Shifted version of $\tilde{\bx}^*$ to the left by $i$ positions \\ \hline 
 $\bu_i $   & Orthogonal component from the $(m,\beta)$-orthogonality of $\{\bs_\ell\}_{\ell \in [m]}$ \\ \hline
 $\br_i $   & Residual component from the $(m,\beta)$-orthogonality of $\{\bs_\ell\}_{\ell \in [m]}$; $\max_{i} \|\br_i\|_2 \le \beta$ \\ \hline
 $e_i$   & The $i$-th noise term for the case of Gaussian noise; $e_i \sim \calN(0,\sigma^2)$ \\ \hline 
 $\omega \in \{0,1\}$ & $\omega = 0$ for the case of random flips and $\omega = 1$ for the case of Gaussian noise\\ \hline 
 $I_{\eta}^c$ & The index set for $i \in [m]$ such that $|\langle \bg,\bu_i\rangle + \omega e_i | > |\langle \bg,\br_i \rangle|$ with high probability \\ \hline
 $I_{\eta}$ & $I_{\eta} = [m]\backslash I_\eta^c$\\ \hline
 $\calE$ & High-probability event regarding $\bg$ and $I_{\eta}$; see~\eqref{eq:event_e}\\ \hline
 $Y_{ij}$ & $Y_{ij} = \tilde{a}_{ij} \theta(\langle\bg,\bu_i\rangle) - \lambda\frac{\bu_i (j-i)}{\|\bu_i\|_2}$  \\ \hline
 $Z_1,Z_2,Z_3$ & Notations used to shorten the expressions in~\eqref{eq:begin_Z123}--\eqref{eq:complex_four_terms3} 
 \\ \hline
 $h_i$ & $h_i = \left\langle \bg,\frac{\bu_i}{\|\bu_i\|_2}\right\rangle \sim\calN(0,1)$ 
 \\ \hline
 $\bar{u}_{j,i,\ell}$ & $\bar{u}_{j,i,\ell} = \frac{\bu_\ell(j-i)}{\|\bu_\ell\|_2} = \mathrm{Cov}[g_{j-i},h_\ell]$  
 \\ \hline
  $\bar{u}_{j,i,\ell}$ & $\bar{u}_{j,i,\ell} = \frac{\bu_\ell(j-i)}{\|\bu_\ell\|_2}$  
 \\ \hline
 $w_{j,i}$ & $w_{j,i} = \sqrt{1- \sum_{\ell \in [m]} \bar{u}_{j,i,\ell}^2}$  
 \\ \hline
 $t_{j,i}$ & A standard Gaussian random variable that is independent of $\{h_\ell\}_{\ell \in [m]}$  
 \\ \hline
 $\theta_{i}$ & For random bit flips, $\theta_i = \theta$; for random noise added before quantization, $\theta_i = \sign(x+\frac{e}{\|\bu_i\|_2}) $ 
 \\ \hline
 $\lambda_i$ & $\lambda_i := \bbE[\theta_{i}(g)g]$ for $g \sim \calN(0,1)$ 
 \\ \hline

\end{tabular}
\caption{A table of notation used in the proof of Lemma~\ref{lem:essential_circ}.} %for the cases of random flips and Gaussian noise.} 
\label{table:notations_lemma1}%The majority of computation cost comes from $x=G(z)$ and $Ax.$}
\end{center}
\end{table*}

% \clearpage
\bibliographystyle{IEEEtran}
\bibliography{techReports,JS_References}

\end{document}